\documentclass[11pt]{article}
\usepackage{amsmath,amsthm, amssymb}
\usepackage{enumerate,color}
\usepackage{graphicx}
\usepackage{caption}
\usepackage{bbm}

\usepackage[top=3cm, bottom=3cm, left=2.5cm, right=2.5cm]{geometry}

\usepackage{dsfont}

\allowdisplaybreaks

\newcommand{\norm}[1]{\left\lVert#1\right\rVert}

\newcommand{\set}[1]{\{#1\}}

\newcommand{\X}{\mathcal{X}}
\newcommand{\R}{\mathbb{R}}
\newcommand{\N}{\mathbb{N}}
\newcommand{\Prob}{\mathds{P}}
\newcommand{\E}{\mathbb{E}}
\newcommand{\Er}{\mathcal{E}}

\newcommand{\T}{\mathcal{T}}
\newcommand{\MT}{\mathbb{T}}
\newcommand{\M}{\mathcal{M}}
\newcommand{\Lm}{\Lambda}

\newcommand{\eps}{\epsilon}

\providecommand{\diam}[1]{\operatorname{diam}_{\rho}(#1)} 
\providecommand{\wh}[1]{\widehat{#1}}
\newcommand{\PP}[1]{\Prob\left[ #1\right] } 
\newcommand{\EE}[1]{\mathbb E\left[ #1\right] } 
\newcommand{\abs}[1]{\left| #1\right| } 

\newtheorem{theorem}{Theorem}
\newtheorem{lemma}[theorem]{Lemma}
\newtheorem{proposition}[theorem]{Proposition}
\newtheorem{definition}[theorem]{Definition}
\newtheorem{assumption}[theorem]{Assumption}
\newtheorem{remark}[theorem]{Remark}
\newtheorem{corollary}[theorem]{Corollary}

\title{Multi-Scale Vector Quantization with Reconstruction Trees}

\author{
 E.~Cecini$^{\top}$, E. De Vito$^{\top}$, L. Rosasco$^{\ddag,\star,\dagger}$ \\
$^\ddag$ DIBRIS, Universit\'a degli Studi di Genova\\
$^{\top}$Dipartimento di Matematica, Universit\`a di Genova\\
$^{\star}$  Istituto Italiano di Tecnologia \\
$^{\dagger}$Massachusetts Institute of Technology
}

\begin{document}
\maketitle

%

\begin{abstract}
We propose and study a multi-scale approach to vector quantization.  We develop an algorithm, dubbed reconstruction trees, inspired by decision trees.  Here the objective is   parsimonious reconstruction of unsupervised data, rather than classification. Contrasted to more standard vector quantization methods, such as $K$-means, the proposed approach leverages a family of given partitions, to quickly explore the data in a coarse to fine --multi-scale-- fashion. Our main technical contribution is an analysis of the expected distortion achieved by the proposed algorithm, when the data are assumed to be sampled from a fixed unknown distribution. In this context, we derive  both asymptotic and finite sample results under suitable regularity assumptions on the distribution. As a special case, we consider the setting  where the data generating distribution is supported on a compact Riemannian sub-manifold. Tools from differential geometry and concentration of measure are useful in our analysis. 
\end{abstract}




\section{Introduction}

Dealing with large high-dimensional data-sets is a hallmark of modern signal processing and machine learning. 
In this context, finding parsimonious representation from unlabeled data
 is often key to both reliable estimation and efficient computations, and more generally for exploratory data analysis. 
A classical approach  to this problem is principal component analysis (PCA), relying on the assumption that data are well represented by a linear subspace.
Starting from PCA a number of developments can be considered to relax the linearity assumption. For example kernel PCA is based on 
performing PCA after a suitable nonlinear embedding \cite{kpca}. Sparse dictionary learning tries to find a set of vectors 
on which the data can be written as sparse linear combinations \cite{olshausen}. Another line of works assumes the data to be sampled 
from a distribution supported on a manifold and includes isomap \cite{tenenbaum}, Hessian eigenmaps \cite{donoho},  Laplacian eigenmaps \cite{belkin2001laplacian} 
and related developments such as diffusion maps \cite{nadler2006diffusion}. A more recent and original perspective has been  proposed in \cite{lima16}, and called 
geometric multi-resolution analysis (GMRA). Here the idea is to borrow and generalize ideas from multi-resolution analysis and wavelet theory \cite{mallat},
to derive locally linear representation organized in a multi-scale fashion. The corresponding algorithm is based on a cascade of local PCAs and is reminding of classical decision trees for function approximation, see e.g. \cite{hastie}. In this paper we further explore the  ideas introduced to GMRA which is our main reference.

Indeed, we consider these ideas in the context of vector quantization, which is another classical, and extreme,  example  of parsimonious representation.
Here, a set of centers and corresponding partition is considered, and
then all data points in each cell of the partition are represented by the corresponding center. The most classical approach  in this context  is probably $k$-means, where  a set of centers (means) is defined by a  non-convex optimization   problem over all possible partitions. 
Our approach offers an alternative to $k$-means, by following the basic idea of GMRA and decision trees, but considering local means rather than local PCA. In this view, our approach can be seen as a zero-th order version of GMRA, hence providing a piece-wise constant data approximation. 
Compared to $k$-means, the search for a partition is performed through a coarse-to-fine  recursive procedure, rather than by global optimization. A strategy that we call reconstruction tree.  As a byproduct the corresponding vector quantization is multi-scale, and  naturally yields a multi-resolution representation of the data. Our main technical contribution is a theoretical analysis of the above multi-scale vector quantization procedure. We consider a statistical learning framework, where the data are assumed to be sampled according to some fixed unknown distribution and measure performance according to the so called expected distortion, measuring the reconstruction error with respect to the whole data distribution. Our main result is deriving corresponding finite sample bound in terms of natural geometric assumptions. 

The rest of the paper is organized as follows. After describing the basic ideas in the context of vector quantization in Section~\ref{sec:vq}, we 
present the algorithm we study in Section~\ref{sec:msvq}. In Section~\ref{sec:ass}, we introduce the basic theoretical assumptions needed in our analysis, and illustrate them considering the case where the data are sampled from a manifold. In Section~\ref{sec:main}, we present and discuss our main results, and detail the main  steps in their proofs. All other proofs are deferred to the Appendix.

\section{Vector quantization \& distortion}\label{sec:vq}
We next  introduce the problem of interest and comment on its connections with  related questions. 

A vector quantization  (VQ) procedure is defined by a set of code vectors/centers and an associated  partition of the data space. The idea is that compression can be achieved replacing all points in a given cell of the partition by the corresponding code vector.
More precisely,  assuming that the data space is $\R^D$, consider a
set of code vectors $c_1\in I_1, \dots, c_k\in I_k,$,
where the set of  cells $\Lambda=\{I_1, \dots, I_k\}$, defines a
partition of $\R^D$. A nonlinear projection $P_\Lambda:\R^D\to \R^D$ can be defined by 
\begin{align}\label{eq:quant}
P_\Lambda(x)=\sum_{j=1}^k c_j \mathbbm 1_{I_j}(x) \qquad x\in \R^D,
\end{align}
where $\mathbbm 1_I$ is the characteristic function of $I$, \textit{ i.e.}
    \[
\mathbbm 1_I(x) =
\begin{cases}
  1 & x\in I \\
  0 & x\notin I
\end{cases} .
      \]
Given a set of points $x_1, \dots, x_n$ the error (distortion) incurred by this  nonlinear projection can be defined as
\[
\wh  \Er[P_\Lambda]= \frac 1 n \sum_{i=1}^n \norm{x_i-P_\Lm(x_i)}^2,
\]
where $\norm{\cdot}$ is the Euclidean norm in $\R^D$. 
If we consider the data to be identical and independent samples of  a random variable $X$ in $\R^D$, then the  following error measure can also be considered
\begin{align}\label{eq:exdis}
 \Er[P_\Lambda]= 
\E[\norm{X-P_\Lm(X)}^2].
\end{align}
The above error measure is the expected  distortion associated to the quantization defined by $P_\Lambda$.

 In the following we are interested in deriving VQ schemes with small expected distortion given a dataset $x_1, \dots, x_n$ of  $n$ samples  of $X$. Before describing the algorithm we propose, we add two remarks.

\begin{remark}[Comparison to supervised learning]\label{rem:ULSL}~~\\
Classical supervised learning is concerned with the problem of
inferring a functional relationship $f$ given a set of input-output
pairs $(x_1, y_1), \dots, (x_n,y_n)$. A classical error measure is the
least squares loss $\norm{y-f(x)}^2$ (if the outputs are vectors valued).  A parallel between the above setting  and  supervised learning   can be seen, considering the case where the input and output spaces coincide and the  least squares loss would be $\norm{x-f(x)}^2$. Clearly, in this case an optimal solution is given by the identity map, unless further constraints are imposed.
Following  the above remark, we can view the nonlinear projection
$P_\Lambda$ as a piece-wise constant approximation of the identity
map, possibly  providing a parsimonious representation.
\end{remark}

\begin{remark}[Vector quantization as Dictionary Learning]
A dictionary is  a set of vectors (called atoms) $a_1, \dots, a_p$ that can be used to approximately decompose  each point $x$  in the data space, i.e.
$
x\approx \sum_{j=1}^p a_i \beta_j, 
$
with  $\beta=(\beta_1, \dots, \beta_p)\in \R^p$ being a coefficients  vector. Given  a set of points $x_1,\dots, x_n$, dictionary learning is the  problem of estimating a dictionary, as well as a set of coefficients vectors.Vector quantization can be viewed as a form of dictionary learning where the code vectors are the atoms, and coefficients vectors  are binary and have at most one non zero component \cite{olshausen}. 
\end{remark}

The above remarks provide different views on the problem under study. We next discuss how  ideas from decision trees in supervised learning can be borrowed to define a novel VQ approach.

\section{Multi-scale vector quantization via reconstruction trees}\label{sec:msvq}

We next describe our approach to Multi-Scale Vector Quantization (MSVQ), based on a recursive procedure that we call reconstruction trees, since it is inspired by 
decision trees for function approximation. The key ingredient in the proposed approach is a family of partitions  organized in a tree. The partition at the root of the tree has the largest cells, while partitions with cells of decreasing size are found in lower leafs. This \textit{ partition tree} provides a multi-scale description of the data space: the lower the leaves, the finer the scale. The idea is to use data to identify  a subset of cells, and corresponding partition, providing a  VQ solution~\eqref{eq:quant} with low expected distortion \eqref{eq:exdis}. We next describe this idea in detail.

\subsection{Partition trees and subtrees}

We begin introducing the definition of a partition tree. 
In the following  we denote by  $\mathcal X\subset\R^D$ the data space
endowed with its natural Borel $\sigma$-algebra $\mathcal B(\X)$ and
by $\sharp A$ the cardinality of a set $A$. 
\begin{definition}
  A partition tree $\mathbb T$ is a denumerable family
  $\{ \Lambda_j \}_{j\in\N}$ of partitions of $\X$ satisfying
  \begin{enumerate}[a)]
  \item $\Lambda_0=\{\X\}$ is the root of the tree;
  \item each family $\Lambda_j=\{ I \}_{I\in \Lambda_j}$ is a finite
    partition of $\X$ of Borel subsets, \textit{ i.e } 
    \[
      \begin{cases}
        \X =\bigcup_{I\in\Lm_j} I \\
        I\cap J =\emptyset\qquad\forall I,J\in \Lm_j, I\neq J \\
        \sharp \Lm_j<+\infty \\
        I \in \mathcal B(\X) \qquad I\in\Lambda_j
      \end{cases}
      ;\]
  \item for each $I\in \Lambda_j$, there exists a family
    $\mathcal{C}(I)\subseteq\Lm_{j+1}$ such that
    \begin{align}
      \begin{cases}
        I = \bigcup\limits_{J\in\mathcal{C}(I)}J \\[4pt]
        \sharp\mathcal{C}(I)\leq a
      \end{cases}\label{eq:22} ,
    \end{align}
    where $a\in (0,+\infty)$ is a constant depending only on $\MT$.
  \end{enumerate}
\end{definition}

Note that, we allow  the partition tree to have arbitrary, possibly  infinite, depth, needed to derive  asymptotic  results. 

Further, notice that, since $\#\Lambda_j\leq a^j$  the constant $a$
characterizes how the  cardinality of each partition  increases at
finer scales. The case $a=2$ corresponds to  dyadic trees. 

We add  some further definitions.  For any $j\in\N$, and $I\in \Lm_j$ the  depth of $I$ is $j$ and is denoted by $j_I$.
The cells in  $\mathcal{C}(I)\subset \Lm_{j+1} $ are  the  children of $I$, the
unique cell  $J\in\Lm_{j-1}$ such that $I\in\mathcal{C}(J)$ is the
parent of $I$ and is denoted by $\mathcal{P}(I)$ (by definition
$\mathcal P(\X)=\X$).   We regard $\mathbb T$ as  a set of nodes where each node is
defined by a cell $I$ with its parent $\mathcal{P}(I)$ and its
children $\mathcal{C}(I)$.  The following definition will be crucial.
\begin{definition}
A (proper) subtree of $\MT$ is a family $\mathcal T\subset\MT$ of
cells  such that $\mathcal P(I)\in\mathcal T$ for all $I\in \mathcal T$
and
\[
\Lm(\mathcal T)=\{ I\in\mathbb T : I\not\in \mathcal T, \mathcal P(I)\in\mathcal T\},
\]
denotes the set of outer leaves. 
\end{definition}
It is important in what follows that   $\Lm(\mathcal T)$ is a
partition of $\X$ if $\T$ is finite, see Lemma~ {\ref{lem:wba}}.

\subsection{Reconstruction trees} \label{subsec:rectrees}

We next discuss a data driven procedure to derive a suitable partition and a corresponding nonlinear projection.  To do  this end, we need  a few definitions depending on an available dataset $x_1, \dots, x_n$. 

For each cell $I$, we fix an arbitrary point $\hat{x}^*_I\in \X$ and
define the corresponding cardinality and center of mass, respectively,
as 
\begin{align}\label{defcents}
   n_I  =\sum_{i=1}^n \mathbbm 1_{I}(x_i), \quad \quad\quad\quad 
   \wh{c}_I  =\begin{cases}
     \dfrac{1}{n_I} \sum\limits_
     {i=1}^nx_i\mathbbm 1_{I}(x_i)& \text{if }x\in I\text{
       and } n_I\neq 0 \\
    \hat{x}_I^* & \text{if }x\in I \text{ and } n_I=0
    \end{cases}.
\end{align}
If $0\in\X$,  a typical choice is $\hat{x}_I^*=0$ for all cells
$I\in\MT$. While $\mathcal E(\hat{P}_n)$  depends on the choice of
  $\hat{x}_I^*$, our bounds hold true for all choices. We point out
  that it is more convenient to choose $\hat{x}_I^*\in I$, as this
  (arbitrary) choice produces an improvement of $\mathcal E(\hat{P}_n)$
  for free, in particular whenever $\mathbb E[X\in I]>0$ but $n_I=0$.

Using this quantity we can define  a local error measure
for each  cell $I$,
\[
\wh{\Er}_I   =\frac{1}{n} \sum_{x_i\in I}\norm{x_i -\wh{c}_I}^2  =  \frac{1}{n} \sum_{i=1}^n \norm{x_i -\wh{c}_I}^2 \mathbbm
1_{I}(x_i),
\]
as well as the potential error difference induced by  considering a refinement, 
\begin{align}\label{eq:littleeps}
\wh{\eps}^2_I   = \wh{\Er}_I  -\sum_{J\in\mathcal C(I)} \wh{\Er}_J =\frac{1}{n} 
\sum_{J\in\mathcal C(I)}  \norm{\wh{c}_J -\wh{c}_I}^2,
\end{align}
where the second equality is consequence of the between-within
decomposition of the variance. 
Following~\cite{binev2005universal}, we first truncate the
partition tree at a given depth, depending on the size of the data
set. More precisely,  given $\gamma>0$, we set
\begin{align}
  \label{eq:57}
  j_n = \left\lfloor\frac{\gamma \ln n}{\ln a} \right\rfloor
  \qquad \Longrightarrow \qquad a^{j_n}\leq n^\gamma.
\end{align}
Deeper trees are considered as data size grows.

As a second step, we fix a threshold $\eta>0$ and select the cells such that
$\hat{\eps}_I\geq\eta$. Since $\hat{\eps}_I$ is not an 
decreasing function with the depth of the tree, this requires some care --
see Remark~\ref{alternative} for an alternative construction.   
Indeed, we define the $\eta$-dependent subtree 
\begin{align}
  \label{eq:46bis}
  \wh{\T}_\eta=
  \begin{cases}
   \set{\X} & \text{ if } \wh{\eps}_I<\eta \quad \forall I\in
   \bigcup\limits_{j\leq j_n}\Lm_j \\[4pt]
    \{ I\in \mathbb T \mid \exists j\leq j_n,\, J\in \Lm_j,\, J\subset I,\,
    \wh{\eps}_J\geq \eta\} & \text{otherwise} 
  \end{cases} 
\end{align}
and $\wh{\Lm}_\eta$ is defined as  outerleaves
of $\wh{\T}_\eta$, \textit{ i.e.} $\wh{\Lm}_\eta=\Lm(\wh{\T}_\eta)$.
Note that $\wh{\T}_\eta$ is finite, so that
by Lemma~\ref{lem:wba} $\wh{\T}_\eta$ is a partition of $\X$ such that 
$j_I\leq (\gamma \ln n)/\ln a$ for all $I\in
\wh{\Lm}_\eta$. 

The code vectors are  the centers of mass of the cells  the above empirical partition, and 
the corresponding nonlinear projection is
\begin{align}\label{eq:15biss}
\wh{P}_\eta =
\sum_{I\in\wh{\Lm}_{\eta}} \hat{c}_I \mathbbm 1_{I}(x) \qquad \wh{\Lm}_\eta=\Lm(\wh{\T}_\eta).
\end{align}
We add a a few comments, the above vector quantization procedure, that we call reconstruction tree,  is recursive and depends on the threshold $\eta$. Different quantizations and corresponding distortions are achieved by different choices of $\eta$. Smaller values of $\eta$ correspond to vector quantization quantizations with smaller distortion. It is a clear that the empirical distortion becomes zero for a suitably small $\eta$ corresponding to having a single point in each cell. Understanding the behaviour of the expected distortion as function of $\eta$ and the number of points is our main theoretical contribution. Before discussing these results we discuss the connection of the above approach to related ideas.
A similar construction is given in \cite{lima16}, where  however  the
thresholding criterion $\eta$ depends on the scale, see Section~2.3 of
the cited reference.

\subsection{Comparison with related Topics}\label{sec:comp}
~\\
The above approach can be compared with a number of different ideas. 
\paragraph{ Decision and Regression Trees}
We call the above procedure Reconstruction Tree, since its definition
is formally analogous to that of decision trees for supervised
learning; see for example \cite{hastie2013introduction}, Chapter 8. In
particular our construction and analysis follows closely that of tree
based estimators studied in \cite{binev2005universal}, in the context
of least square regression.  As commented in Remark~\ref{rem:ULSL} our
problem can actually be interpreted as a special instance of vector
valued regression
where  the regression function $f_\rho$ is the identity function from
the input space to itself (regarded as the output space): indeed, referring to the notation of
\cite{binev2005universal}, the conditional distribution of the output $y$
given $x$ is the Delta measure at $x$. 
From this point on the two formalisms overlap, in that $\norm{f-f_\rho}_{L_2}^2$ becomes $\E[\norm{X-f(X)}^2]$. Since the tree based estimators $f$ considered in \cite{binev2005universal} are piece-wise constant, the expected square loss cannot vanish and its analysis is non trivial.
Despite the formal similarity, the two settings do exhibit distinct features. For example, the analysis in \cite{binev2005universal} is specifically formulated for scalar functions, while our analysis is necessarily vectorial in nature. In \cite{binev2005universal} a uniform bound $|y|<M$ is imposed, while in our setting we can assume a local bound for free; namely, if $f$ is constant on a cell $I\subset\X$ then $\norm{x-f(x)}^2\leq {\operatorname{diam}(I)}^2$ for all $x\in I$. The present setting finds a natural instance in the case of a probability measure supported on a smooth manifold isometrically embedded in $\X$ (see Section~\ref{sec:ass}), while this case is hardly addressed explicitly in the literature about least square regression.  For example, the manifold case is actually discussed, in the context of classification through Decision Trees, in \cite{scott2006minimax}. One last point is that the present work contains explicit quantitative results about the approximation error, see Section~\ref{discapprx}, while similar results are not available in the setting of \cite{binev2005universal}, that aspect being typically addressed indirectly in the corresponding literature.

\paragraph{ Empirical risk minimization}
Again in analogy to supervised learning, as in \cite{binev2005universal}, one can consider the minimization problem: 
\begin{alignat*}{1}
\min_{F\in \mathcal H} \wh  \Er[F],  \end{alignat*}
where $\mathcal H$ is the (finite-dimensional) vector space of the
vector fields $F:\X\to\R^D$, which are piecewise constant on
a given partition $\Lm$.  There correspond a number of independent
minimization problems, one for each cell in $\Lm$, so that $\wh
P_\eta$ from \eqref{eq:15biss} is easily shown to be a minimizer. The
minimizer is not unique, since the value of $F$ is irrelevant on cells
$I\in\Lm$ such that $n_I=0$. Similar considerations hold for
$\min_{F\in \mathcal H}   \Er[F]$ as well, in which case the value of
$F$ is irrelevant whenever $\E[X\in I]=0$. See also
Section~\ref{subsec:rectrees}, Section~\ref{discapprx} and
Lemma~\ref{rem:choice}. 

One could consider minimization over a wider class of functions,
piece-wise constant on different partitions, for example on all the
partitions with a given number of cells that are induced by proper
subtrees $\T$ of a given partition tree $\MT$. This would be a
combinatorial optimization problem. The algorithm defined
in~\eqref{eq:46bis} overcomes this issue by providing a one-parameter
coarse-to-fine class of partitions, such that each refinement carries
local improvements $\wh \epsilon$ that are uniformly bounded. As
observed in \cite{binev2005universal}, such a strategy is inspired by
wavelet thresholding. 

\paragraph{ Geometric multi-resolution analysis (GMRA)}
A main motivation for our work is the algorithm GMRA
\cite{allard2012multi,maggioni2014dictionary,lima16}, which introduces
the idea of learning multi-scale dictionaries by geometric
approximation. The main difference between GMRA and Regression Trees
is that the former represents data through a piece-wise linear
approximation, while the latter through a piece-wise constant
approximation. More precisely, rather than considering the center of
mass of the data in each cells~\eqref{defcents}, a linear
approximation is obtained by (local) Principal Component Analysis, so
that the data belonging to a cell are sent to a linear subspace of
suitable dimension, the latter approach being particularly natural in
the case of data supported on a manifold. Another difference is in the
thresholding strategy: unlike \cite{binev2005universal} and our work,
in \cite{lima16} the local improvement $\wh \epsilon$ is scaled
depending on its depth in the tree. One of our purposes is to check
if these \textit{simpler} choices affect the learning rates
significantly. We provide more quantitative comparisons later in
Section~\ref{sec:main}. 

\paragraph{ Wavelets} A main motivation for GMRA is extending ideas
from wavelets and multi-resolution analysis to the  context of machine
learning, where, given the potential  high dimensionality, non-regular
partitions need be considered; this point is discussed in
\cite{allard2012multi,maggioni2014dictionary,lima16} and references
therein. Indeed partition trees generalize the classic notion of
dyadic partitions. In this view, given the piece-wise constant nature
of reconstruction trees, a parallel can be drawn between the latter
and classical Haar wavelets.

\paragraph{ $k$-Means} Our procedure being substantially a vector quantization algorithm, a comparison with the most
common approach to vector quantization, namely $k$-means, is in order. In $k$-means, a set of $k$~code vectors $c_1, \dots, c_k$ are  derived  from  the data and used to define corresponding partitions via the corresponding  Voronoi diagram
\[
V_j= \{x\in \R^D ~|~ \norm{x-c_j}\le \norm{x-c_i}, \forall i=1,\dots, k,~ i\neq j\}.
\]
Code vectors are defined by the minimization of the following empirical objective
\[
\min_{c_1, \dots, c_k}\frac 1 n \sum_{i=1}^n \min_{j=1, \dots, k} \norm{x_i-c_j}^2.
\]
This minimization problem is non convex and is typically solved by
alternating minimization, a procedure referred to as Lloyd's algorithm
\cite{lloyd1982least}. The inner iteration assigns each point to a
center, hence a corresponding Voronoi cell. The output minimization
can be  easily shown to update the code vectors by computing the
center of mass, \textit{ the mean}, of each Voronoi cell. In general
the algorithm is ensures to decrease or at least not increase the
objective function and to converge in finite time to a local
minimum. Clearly, the initialization is important, and initializations
exist yielding some stronger convergence guarantees. In particular,
$k$-means++ is a random initialization providing on average an
$\epsilon$-approximation to the global minimum \cite{arthur2007k}. \\ 
Compared to $k$-means, reconstruction trees restrict the search for a
partition over a prescribed family defined by the partition tree. In
turns, they allow a fast multi-scale exploration of the data, while
$k$-means requires solving a new optimization problem each time $k$ is
changed. Indeed it can be shown that a solution for the $(k-1)$-means
problem leads to a bad initialization for the $k$-means problem. In
other words, unlike restriction trees, the partitions found by
$k$-means at different scales (different values of $k$)  are generally
unrelated, and cannot be seen as refinements of one another.

\paragraph{ Hierarchical clustering}  Lastly, our coarse-to-fine
approach can be compared with hierarchical clustering, in particular
[with the so called Ward's method, which proceeds in the opposite
way. Indeed, this algorithm produces a coarser partition of the data
starting from a finer. It starts with a Voronoi partition having all
the data as centers, and at each step it merges a couple of cells that
have the smallest so called between cluster inertia
\cite{ward1963hierarchical}. Interestingly this definition has an
analogue in our algorithm. Our $\wh{\Er}_I$ corresponds to the within
cluster inertia of a cell $I$ while $\wh{\eps}^2_I$ to the between
cluster inertia (up to a factor $1/n$) of cells that merge into
$I$. Nevertheless the obtained partitions will not in general
coincide, unless very specific choices are made ad hoc.

\section{General assumptions and manifold  setting}\label{sec:ass}

In this section, we introduce our main  assumptions and then
discuss a motivating example   where  data are sampled at random from a manifold.

We consider a statistical learning framework, in the sense that we assume the data to be  random samples from an underlying probability measure.
More precisely, we assume the available data to be a realization of  
$n$ identical and independent  random vectors $X_1,\ldots,X_n$ taking
values in a bounded subset $\X\subset\R^D$  and we denote by $\rho$
the common law.  Up to a rescaling  and a translation, we assume that
$0\in\X$ and  
\begin{align}
\operatorname{diam}(\X)=\sup_{x,y\in\X} \norm{x-y} \leq
1.\label{eq:11b}
\end{align}

Our main assumption
relates the distribution underlying the data to the partition tree to
be used to derive a MSVQ via reconstruction trees.   To state it, we recall 
the notion of essential diamater of a cell  $I$, namely
\begin{alignat*}{1}
  \diam{I}  & =  \inf_{{J\subset I}\atop{\rho(J)=0}} \operatorname{diam}(I\setminus J).
\end{alignat*}

\begin{assumption}\label{apriori}
There exists  $s>0$ and $b>1$ such that  for all $I\in\mathbb T$
\begin{subequations}
  \begin{alignat}{1}
      \diam{I}  & \leq C_1 {\rho(I)^s }  \label{eq:1}  \\
      \diam{I}   & \leq C_2 b^{-j_I} \label{eq:6}
  \end{alignat}
\end{subequations}
where $C_1>0$ and $C_2>0$ are fixed constants depending only on $\MT$. 
\end{assumption}
To simplify the notation, we write $c_{\MT}$ for a constant 
depending only on $s,b,C_1, C_2$ and we write $A\lesssim B$ if there
exists a constant $c_{\MT}>0$  such that $A\leq c_{\MT} B$.

Given the partition tree $\MT$, the parameters $s$ and $b$ define a class
$\mathcal P_{b,s}(\MT)$ 
of probability measures $\rho$ and for this class we are able to
provide a finite sample bound on the distortion error of our estimator
$\hat{P}_\eta$, see~\eqref{eq:50}. In the context of supervised machine
learning  $\mathcal P_{b,s}(\MT)$  is an a-priori class of
distributions defining a upper learning rate, see~\eqref{eq:56}. It is an open problem to
provide a lower min-max learning rate. 

Clearly,~\eqref{eq:6} is implied by the distribution-independent
assumption 
\begin{align}
    \label{eq:6tris}
    \operatorname{diam}(I)   \lesssim b^{-j_I} \qquad \text{ for all }  I\in \MT,
  \end{align}
\textit{ i.e.} the diameter of the cells goes to zero exponentially
with their depth. This assumption ensures that the reconstruction
error goes to zero and, in supervised learning, it  corresponds to the
assumption that the hypotheses space is rich enough to approximate any
regression function, compare with condition~(A4) in \cite{lima16}. 

Eq.~\eqref{eq:1} is a sort of regularity condition on the shape of the
cells and, if it holds true,~\eqref{eq:6}
is implied by the following condition 
\begin{align}
    \label{eq:6bis}
    \rho(I)  \lesssim c^{-j_I} \qquad \text{ for all }  I\in \MT,
  \end{align}
which states that the volume of the cells goes to zero exponentially
with their depth.

In \cite{lima16}, following ideas from \cite{binev2005universal}, it
is introduced a suitable model class, see Definition~5, in terms
of the decay of the approximation error, compare Eq.~(7) of
\cite{lima16} with~\eqref{eq:14} below. This important point is further discussed in Section~\ref{discapprx}.

In many cases the parameter $s$ is related to the intrinsic dimension of the
data. For example, if $\X=[0,1)^D$ is the unit cube and $\rho$ is
given by 
\[
  \rho(E) = \displaystyle{\int\limits_{E} p(x) dx }  \qquad E\in\mathcal B(\X) ,  \]
where $dx$ is the Lebesgue measure of $\R^D$ and the density $p$  is bounded from above
and away from zero, see~\eqref{eq:71b} below,  it is easily to check
that the family 
$\MT=\set{\Lambda_j}$ of dyadic cubes
\[
\Lambda_j =\set{[2^{-j}(k_1-1), 2^{-j}k_1)\times\ldots\times [2^{-j}(k_D-1), 2^{-j}k_D)  \mid k_1,\ldots,k_D=1,...,\ldots,2^j}\qquad  j\in\N
  \]
is  a partition tree satisfying Assumption~\ref{apriori} with $s=1/D$ and a
suitable $b>1$.  The construction of dyadic cubes can be extended to
more general  settings, see \cite{christ,gile17} and references
therein, by providing a large class of other
examples, as shown by the following result. The proof is deferred to
Section~\ref{mainproof}. 

\begin{proposition}\label{manifold}
Assume that the support $\M$ of $\rho$  is a connected
submanifold  of $\R^D$ and the  distribution $\rho$ is given by
\begin{subequations}
  \begin{alignat}{2}
    & \rho(E) = \displaystyle{\int\limits_{E\cap\M} p(x)
      d\rho_{\M}(x) \label{eq:71a} }&&\qquad
    E\in\mathcal B(\X) \\
    & 0<p_1 \leq p(x)\leq p_2<+\infty &&\qquad x\in\M ,\label{eq:71b}
  \end{alignat}
\end{subequations}
where  $\rho_\M$ is the Riemannian volume element of $\M$, then there exists a
partition tree $\MT$ of $\X$ satisfying~Assumption~\ref{apriori} with
$s=1/d$, where $d$ is the intrinsic dimension of  $\M$.
\end{proposition}
We recall that, as a submanifold of $\R^D$, $\M$ becomes a compact
Riemannian manifold with Riemannian distance $d_\M$ and Riemannian
volume element $\rho_{\M}$. We stress that the construction of the
dyadic cubes only depend on $d_\M$.  Proposition~\ref{manifold} has to
be compared with Proposition~3 and Lemma~6 in \cite{lima16}.

By inspecting the proof of the above result, it is possible to show
that a partition tree satisfying 
Assumptions~\ref{apriori} always exists if  there are a metric
$d$ and a Borel measure $\nu$ on  $\M$ such that $(\M,d,\nu)$ is an
Ahlfors regular metric  measure \cite[p. 413]{gromov}, clg$\rho$ has density $p$ with respect to $\nu$
satisfying~\eqref{eq:71b} and the embedding of $(\M,d)$ into
$(\R^d,\norm{\cdot})$ is a Lipschitz function. 

\section{Main result}\label{sec:main}
In this section we state and discuss our main results, characterizing the expected distortion of reconstruction trees.
The proofs are deferred to Section~\ref{mainproof}.
Our first result is a probabilistic bound for any given  threshold
$\eta$. Recall that $s>0$ is defined by ~$\eqref{eq:1}$ and
$\hat{P}_\eta$ by~\eqref{eq:46bis} and~\eqref{eq:15biss}.
\begin{theorem}\label{main}
Fix $\gamma>0$ as in~$\eqref{eq:57}$ and $\eta>0$, for any $0<\sigma<s$
\begin{align}
  \label{eq:50}
 \PP{\mathcal E[\wh{P}_{{\eta}}]\gtrsim \eta^{\frac{4\sigma}{2\sigma+1}}
   (1+t)} \lesssim  \eta^{-\frac{2}{2\sigma+1}} \exp\left( -c_{\MT}  n\eta^2 t
\right) +(n^\gamma+ \eta^{-\frac{2}{2\sigma+1}} )   \exp\left( -c_a
  n\eta^2\right)\qquad t>0, 
\end{align} 
where $c_a=\frac{1}{128 (a+1)}$ and $c_{\MT}>0$ depends on the partion tree $\MT$. 
\end{theorem}
As shown in Remark~\ref{log}, it is possible to set $\sigma=s$ up to an extra
logarithmic factor. 

Next, we show how it allows derive a
choice for $\eta$ as a function of the number of examples, and a
corresponding expected distortion bound.
  
\begin{corollary}\label{pasquale}
Fix $\gamma>1$, $\beta>0$ and set 
\begin{align}
  \label{eq:55}
  \eta_n = \sqrt{ \frac{(\gamma+\beta)\ln n}{c_an}} \qquad\text{and}\qquad \wh{P}_n =
  \wh{P}_{{\eta_n}} \qquad n\ge 1,
\end{align}
where $c_a=\frac{1}{128 (a+1)}$. Then for any $0<\sigma<s$ 
\begin{subequations} 
  \begin{align}
    \label{eq:56}
    \PP{\mathcal E[\wh{P}_{n}] \gtrsim
      \left(\frac{\ln n}{n}\right)^{\frac{2\sigma}{2\sigma+1}} (1+t) }
    \lesssim  \frac{1}{n^\beta} + \frac{1}{n^{
\overline{c}_\MT t-1}},
  \end{align}
where $\overline{c}_\MT>0$ is a constant depending on the partition
tree $\MT$.  Furthermore
\begin{align}
  \label{eq:69}
  \lim_{t\to+\infty} 
\limsup_{n\to+\infty} \sup_{\rho\in \mathcal    P_{b,s}(\MT)} \PP{\mathcal
  E[\wh{P}_{n}] \gtrsim      \left(\frac{\ln
      n}{n}\right)^{\frac{2\sigma}{2\sigma+1}} t} = 0  ,
\end{align}
where $\mathcal P_{b,s}(\MT)$  is the family of distributions $\rho$ such that
Assumptions~$\ref{apriori}$ hold true.  
\end{subequations}
\end{corollary}
If $t$ is chosen large enough so that $\overline{c}_\MT t-1=\beta$,
then bound~\eqref{eq:56} reads as
\[
\PP{\mathcal E[\wh{P}_{n}] \geq  c_1
      \left(\frac{\ln n}{n}\right)^{\frac{2\sigma}{2\sigma+1}} }
    \leq c_2\frac{1}{n^\beta} 
  \]
where $c_1$ and $c_2$ are suitable constants depending on $\MT$. This
bound  can be 
compared with Theorem~8 in 
\cite{lima16} under the assumption that $\M$ is a compact $C^\infty$
manifold. Eq.~\eqref{eq:56}  with $s=1/d$ gives a
convergence rate of the order $(\ln n/n)^{\frac{2p}{2p+ d}}$ for any
$p=\sigma/s\in (0,1)$, whereas the GMRA
algorithm has a rate of the order  $(\ln n/n)^{\frac{2}{2+ d}}$, see
also Proposition~3 of \cite{lima16}. Hence, up to a logarithmic factor
our estimator has the same convergence rate of the GMRA
algorithm. However it is in order to notice that our algorithm works with a cheaper representation; indeed, given the
adaptive partition $\hat{\Lambda}_\eta$, it only requires to compute and store the centers of mass $\{\hat{c_I}\}_{I\in\hat{\Lambda}_\eta}$.

In a similar setting, in \cite{canas2012learning}, it is shown that the $k$-means
algorithm with a suitable choice of $k=k_n$ depending on $n$,
provides has a convergence rate of the order $(1/n)^{\frac{1}{d+1}}$
  and $k$-flat algorithm of the order $(1/n)^{\frac{2}{d+4}}$.

The proof of Theorem~\ref{main} relies on splitting the error in
several terms. In particular, it requires studying the stability to
random sampling and the approximation properties of reconstruction
trees. This latter result is relevant in the context of quantization
of probability measures, hence of interest in its own right. We
present this result first.

Towards this end, we need to introduce the infinite sample version of the reconstruction tree. 
  For any cell $I\in\mathbb T$,  denote the volume of the cell by
\[
\rho_I  = \rho(I),
\]
 the center of mass of the cell by
\[
c_I = \begin{cases}
  \frac{1}{\rho_I}\int\limits_I x \, d\rho(x)  & \text{ if } \rho_I>0 \\[4pt]
    x_I^* & \text{ if } \rho_I=0,
\end{cases}  
\]
where $x^*_I$ is an arbitrary point in $\X$. The local expected distortion in a cell by
\[
\Er_I  =\int\limits_I \norm{x-c_I}^2 d\rho(x),
\]
and 
\[
\epsilon^2_I =\Er_I-\sum_{J\in\mathcal{C}(I)}\Er_J =
\sum_{J\in\mathcal C(I)} \rho_I \norm{c_J-c_I}^2 .
\] 
Given the threshold $\eta>0$, define the subtree 
\begin{align}
  \label{eq:46tris} 
  {\T}_\eta=
  \begin{cases}
   \set{\X} & \text{ if } {\eps}_I<\eta \quad \forall I\in \MT\\[4pt]
    \{ I\in \mathbb T \mid \exists J\in\MT \text{ such that }
    J\subset I \text{ and } 
    {\eps}_J\geq \eta\} & \text{otherwise} 
  \end{cases} ,
\end{align}
and let ${\Lm}_\eta=\Lm({\T}_\eta)$ be the corresponding
outerleaves. Lemma~\ref{lem:finite} shows that $ {\T}_\eta$ is finite,
so that by Lemma~\ref{lem:wba} ${\Lm}_\eta$ is a partition and the
corresponding  nonlinear projection is 
\begin{alignat}{2}\label{eq:15}
  P_{\Lm_\eta} (x) & =\sum_{I\in\Lm_\eta} c_I \mathbbm 1_{I}(x) 
\end{alignat}
so that the code vectors are the centers of mass of the
cells.

Comparing the definition of $\T_\eta$ and $\hat{\T}_\eta$, we
  observe that $\hat{\T}_\eta$ is truncated at the depth $j_n$ given
  by~\eqref{eq:57}, whereas $\T_\eta$ is not truncated, but its
  maximal depth is bounded by Lemma~\ref{max_depth}.

Given the above definitions, we have the following result.
\begin{proposition}\label{apprx}
Given $\eta>0$, for all $0<\sigma<s$
\begin{align}
  \label{eq:14}
   \mathcal E(P_{\Lm_{\eta}}) \lesssim\, \eta^{\frac{4\sigma}{2\sigma+1}}.
\end{align}
\end{proposition}
Note that the bound is meaningful only if $0<\eta<1$. Indeed for
$\eta\geq 1$ $\Lm_{\eta} =\set{\X}$ and $\mathcal
E(P_{\Lm_{\eta}})\leq 1$,  see Remark~\ref{eta}. 

\subsection{Approximation Error}\label{discapprx} 
 ~\\
The quantity $\mathcal{E} (P_{\Lm_{\eta}})$ is called approximation error, by analogy with the corresponding definition in statistical learning theory, and it plays a special role in our analysis.

The problem of approximating a probability measure with a cloud of points is related to the so called optimal quantization \cite{gruber2004optimum}. The cost of an optimal quantizer is defined as:
\[
V_{N,p}(\rho):= \inf_{S\subset \X, |S|=N}\E[ d(X,S)^p ],
\]
where $d(x,S)=\min_{y\in S} \|x-y\|$. An optimal quantizer corresponds
to a set $S$ of $N$ points attaining the infimum, with the
corresponding Voronoi-Dirichlet partition of $\X$. One can interpret
the approximation error $\mathcal{E} (P_{\Lm_{\eta}})$ as the
quantization cost associated with the (suboptimal) quantizer given by
the partition ${\Lm}_\eta$ as defined in \ref{eq:46tris} with the
corresponding centers $\{c_I\}_{I\in\Lm_\eta}$, and $N:=\#
{\Lm}_\eta$. 

This point of view is also taken through the analysis of $k$-means
given in \cite{canas2012learning}, optimal quantizers corresponding in
fact to absolute minimizers of the $k$-means problem. Asymptotic
estimates for the optimal quantization cost are available, see
\cite{canas2012learning} and references therein. In the case
of $\operatorname{supp} (\rho) = \mathcal{M}$, being $\mathcal{M}$ a
smooth $d$-dimensional manifold isometrically emdedded in $\X$, they
read: 
\begin{align}\label{optapprx}
\lim_{N \to \infty} N^{2/d} V_{N,2}(\rho)= C(\rho),
\end{align}
where $C(\rho)$ is a constant depending only on $\rho$.
We underline that the result provided by Proposition~\ref{apprx} is
actually a non-asymptotic estimate for the quantization cost, when the
quantizer is given by the outcome of our algorithm. The quantization
cost is strictly higher than the optimal one, since, for instance, an
optimal quantizer always corresponds to a Voronoi-Dirichlet partition
\cite{gruber2004optimum}. Nevertheless, as observed in
Section~\ref{sec:comp}, a Voronoi quantizer is not suitable for multiscale
refinements, whereas ours is. 
Proposition~\ref{apprx} does not directly compare with
\eqref{optapprx}, as it depends on a different parameter quantifying
the complexity of the partition, namely $\eta$ instead of $N$. Though,
by carefully applying \eqref{eq:13}, in the manifold case we get: 
\[
\mathcal{E} (P_{\Lm_{\eta}}) \lesssim \left(\frac{\log{N}}{N}\right)^\frac{2}{d}
\]
so that the bound is in fact optimal up to a logaritmic factor.
Furthermore, it is in order to observe that Assumption~\ref{apriori}
together with Proposition~\ref{apprx} provide a more transparent
understanding of the approximation part of the analysis, as compared
to what is provided in \cite{binev2005universal} and
\cite{lima16}. Therein, the approximation error is essentially
addressed by defining the class of probability measures
$\mathcal{B}_s$ as those for which a certain approximation property
holds; see Definition~5 in \cite{lima16} and Definition~5 in
\cite{binev2005universal}, in both being the thresholding algorithm
explicitly used. On the other hand Assumption~\ref{apriori} does not
depend on the thresholding algorithm, but only on the mutual
regularity of $\rho$ and $\MT$. Lastly we notice that, while for the
sake of clarity none of the constants appear explicitly in our
results, the proofs allow in principle to estimate them.

\section{Proofs}\label{mainproof}
In this section we collect some of the proofs of the above results. The more technical proofs are postponed to the appendix. 

\begin{proof}[Proof of Thm.~\ref{manifold}]
We first observe that it is enough to
show that there exists a  partition tree $\MT'=\set{\Lambda_j}$ for
$\M$. Indeed, by adding to each partion $\Lambda_j$, the cell
$I_0=\X\setminus \M$, we get a partion of $\X$, which
satisfies~Assumptions~\ref{apriori},  since $\diam{I_0}=0$. 

Since $\X$ is bounded, then $\M$ is a
  connected compact manifold 
  and, hence, $(\M,d_{\M},\rho_{\M})$ is an Ahlfors regular metric measure space
  \cite[p. 413]{gromov}, \textit{ i.e.}
  \[
d_1\, r^d \leq \rho_{\M}(B_{\M}(x,r))\leq d_2\, r^d  \qquad r\leq\operatorname{diam}(\M),
  \]  
where $B_{\M}(x,r)$ is the ball of center $x$ and radius $r$ with
respect to the Riemannian metric $d_{\M}$. By~\eqref{eq:71b}
  \begin{align}
{d_1 p_1}\, r^d \leq \rho(B_{\M}(x,r))\leq {d_2 p_2}\, r^d  \qquad
r\leq\operatorname{diam}(\M),\label{eq:11}
\end{align}
where $d$ is the intrinsic dimension of $\M$. 
Since $(\M,d_{\M},\rho)$ is an Ahlfors regular metric measure, too, there exists a family
  of dyadic cubes, \textit{ i.e}  for each $j\in\mathbb Z$ there is a family
$\Lambda_j=\set{I}$ of open subsets of $\M$ such that
\begin{subequations}
  \begin{alignat}{1}
     & \rho( \M\setminus \cup_{I\in\Lambda_j} I) = 0  \label{eq:71aa} \\
     & I\cap J =\emptyset  \qquad I,J\in\Lambda_j,\ I\neq J \label{eq:71bb} \\
    &  \text{\rm either } I\cap J =\emptyset  \text{ \rm  or } J\subset I \qquad
    I\in\Lambda_j,\ J\in\Lambda_{j+\ell}   \label{eq:71c} \\
    &  I \supset B_{\M}(x_I, r_0 \delta^j)\qquad I\in \Lambda_j \label{eq:71d} \\
   &     I \subset B_{\M}(x_I, r_1 \delta^j) \qquad I\in \Lambda_j \label{eq:71e} 
  \end{alignat}
\end{subequations}
where $0<r_0<r_1$ and $\delta\in (0,1)$ are given constants
\cite[Thm.~11]{christ}.  As noted in \cite{gile17}, it is always possible to
redefine each cell $I\in\Lambda_j$ by adding a suitable portion of
its boundary in such a way that
\begin{align}
  \label{eq:73}
  \M=\cup_{I\in\Lambda_j} I
\end{align}
and~\eqref{eq:71a}--\eqref{eq:71e} still hold true, possibly with
different constants. 
Since $\M$ is compact, there exists $j_0\in\mathbb Z$ such that
$B_{\M}(x_0,r_1\delta^{j_0})= \M$ for some $x_0\in\M$. Hence,
possibly redefining $j$, $r_0$  and $r_1$,  we can assume that
$\Lambda_0=\set{\M}$ and, as a consequence
of~\eqref{eq:71bb},~\eqref{eq:71c} and~\eqref{eq:73}, the family
$\set{\Lambda_j}_{j\in\N}$ is a  partion tree for $\M$ where the bound
in~\eqref{eq:22} is a consequence of the following standard volume
argument.  Fix $j_0$ large enough such that for all $j\geq  j_0$,  $r_1
\delta^j\leq \operatorname{diam}{\M}$, then given $I\in \Lambda_{j}$
\[
\rho(I)=\sum_{J\in\mathcal C(I)} \rho( J) \geq \sum_{J\in\mathcal
  C(I)}  \rho(B_{\M}(x_J,r_0\delta^{j+1})) \geq \sharp\mathcal C(I)\, 
{d_1 p_1} r_0^d
\delta^{d(j+1)} ,
\]
where the third and the forth inequalities  are  consequence
of~\eqref{eq:71d} and~\eqref{eq:11}. 

On the other hand, by~\eqref{eq:71e} and~\eqref{eq:11}, 
  \[
\rho(I) \leq \rho(B_{\M}(x_I,r_1\delta^j)) \leq {d_2 p_2} r_1^d\delta^{jd},
\]
so that
\[
\sharp\mathcal C(I)\leq \frac {{d_2 p_2}r_1^d}{{d_1 p_1}r_0^d\delta^d}=D.
\]
Bound~\eqref{eq:22} holds true by setting
\[
a= \max\set{\max_{{j<j_0}\atop{I\in\Lambda_J}}\set{ \#\,\mathcal C(I)}, D }.
  \]
We now show that~\eqref{eq:6} holds true.  Indeed, 
since $\M$ is  Riemmannian submanifold of $\R^D$ it holds that
  \begin{align}
    \label{eq:59}
    \norm{y-x}\leq d_{\M}(y,x) \qquad x,y\in \M, 
  \end{align}
see \cite[Cor.~2, Prop.~21,
Chapter~5]{petersen}. Given $I\in\Lambda_j$, by~\eqref{eq:71e}, 
\[ \diam{I} \leq \operatorname{diam}(I)\leq
 \sup_{x,y\in I }
  \norm{x-y} \leq \sup_{x,y\in B_{\M}(x_I,r_1\delta^j) }
  d_{\M}(x,y) \leq 2 r_1 \delta^j ,\] 
so that~\eqref{eq:6} holds true with 
$b=1/\delta>1$ and $C_2=2r_1$. To show~\eqref{eq:1}, given
$I\in\Lambda_j$, by~\eqref{eq:71d} and~\eqref{eq:11} 
\[\rho(I) \geq \rho(B_{\M}(x_I, r_0 \delta^j)) \geq {d_1 p_1} \delta^{jd}.\]
Hence 
\[
\diam{I} \leq 2r_1 \delta^j \leq  2r_1 \left(\frac{\rho(I)}{{d_1 p_1}}\right)^{\frac{1}{d}},
\]
so that~\eqref{eq:1} holds true with $C_1= 2r_1(d_1 p_1)^{-\frac{1}{d}} $ and $s=1/d$.
\end{proof}

The proof of Theorem~\ref{main} borrows ideas from
\cite{binev2005universal,lima16} and combines a number of intermediate
results given in Appendix~\ref{sec:proofs}. For sake of clarity let
$\hat{P}_\eta= \wh{P}_{\wh{\Lm}_{\eta}}$. 
\begin{proof}[Proof of Thm.~$\ref{main}$] 
  Consider the following decomposition
  \begin{alignat*}{1}
    x- \wh{P}_{\wh{\Lm}_{\eta}}(x) & = \left(x -P_{\Lm(\hat{\T}_\eta
        \cup \T_{2\eta})}(x)\right) + \left( P_{\Lm(\hat{\T}_\eta \cup
        \T_{2\eta})}(x) - P_{\Lm(\hat{\T}_\eta \cap
        \T_{\eta/2})}(x)\right) + \\
    & + \left( P_{\Lm(\hat{\T}_\eta \cap \T_{\eta/2})}(x) -
      \wh{P}_{\Lm(\hat{\T}_\eta \cap \T_{2\eta})}(x)\right) + \left(
      \wh{P}_{\Lm(\hat{\T}_\eta \cap \T_{2\eta})}(x) -
      \wh{P}_{\Lm(\hat{\T}_\eta)}(x)\right),
  \end{alignat*}
  which holds for all $x\in\X$. Since
  \[ \norm{\sum_{i=1}^4 v_i}^2 \leq 4 \sum_{i=1}^4 \norm{v_i}^2\qquad
    v_1,\ldots,v_4\in\R^D,\] it holds that
  \begin{alignat*}{1}
    \mathcal E[\wh{P}_{\wh{\Lm}_{\eta}}] & \lesssim
    \underbrace{\mathcal E[P_{\Lm(\hat{\T}_\eta \cup
        \T_{2\eta})}]}_{\text{A}} + \underbrace{\int\limits_\X
      \norm{P_{\Lm(\hat{\T}_\eta \cup \T_{2\eta})}(x)
        - P_{\Lm(\hat{\T}_\eta \cap \T_{\eta/2})}(x) }^2 d\rho(x) }_{\text{B}}\\
    & \quad + \underbrace{\int\limits_\X \norm{P_{\Lm(\hat{\T}_\eta
          \cap \T_{\eta/2})}(x) - \wh{P}_{\Lm(\hat{\T}_\eta \cap
          \T_{\eta/2})}(x) }^2 d\rho(x)
    }_{\text{C}} 
    + \underbrace{\int\limits_\X
      \norm{\wh{P}_{\Lm(\hat{\T}_\eta \cap \T_{2\eta})}(x) -
        \wh{P}_{\Lm(\hat{\T}_\eta)}(x) }^2 d\rho(x)}_{\text{D}} .
  \end{alignat*}
  We bound the four terms.
  \begin{enumerate}[A)]
  \item Since $\hat{\T}_\eta \cup \T_{2\eta} \supset \T_{2\eta}$,
    $\Lm(\hat{\T}_\eta \cup \T_{2\eta})$ is a partition finer than
    $\Lm_{2\eta}$, then
    \[ \mathcal E[P_{\Lm(\hat{\T}_\eta \cup \T_{2\eta})}]\leq \mathcal
      E[P_{\Lm_{2\eta}}] \lesssim \eta^{\frac{4\sigma}{2\sigma+1}},\]
where the last inequality is a consequence of~\eqref{eq:14}.
  \item[B)] Bound~\eqref{eq:24} implies that the term $B$ is zero with
    probability greater than $1-p_{B}$, where
    \[
      p_{B} \lesssim (n^\gamma+\eta^{-\frac{2}{2\sigma+1}})\exp(-c_a
      n\eta^2).
    \]
  \item[C)] Since
    $\Lm(\hat{\T}_\eta \cap \T_{\eta/2})\subset \T_{\eta/2}\cup
    \Lm_{\eta/2}=\mathcal I $ and
    $\sharp \Lm(\hat{\T}_\eta \cap \T_{\eta/2})\leq \sharp
    \Lm_{\eta/2}=N$, by~\eqref{eq:20} term C is bounded by
    $t^*=\eta^{\frac{4\sigma}{2\sigma+1}} t$ with probability greater
    than $1-p_C$ with
    \begin{align}
      \label{eq:52}
      p_C=2 \sharp\mathcal I \exp\left( -
        \frac{nt^*}{4 N} \right) \lesssim   \eta^{-\frac{2}{2\sigma+1}}
      \exp\left( -c_{\MT} n \eta^2 t\right) 
    \end{align}
    where the second inequality is a consequence of~\eqref{eq:13bis}
    and~\eqref{eq:13}, and $c_{\MT}>0$ is a suitable constant
    depending on the partition tree $\MT$. 
  \item[D)] By~\eqref{eq:49} term D is zero with probability greater
    that $1-p_D$ where
    \[p_D \lesssim \eta^{-\frac{2}{2\sigma+1}}\exp(-c_a n\eta^2) .
    \]
  \end{enumerate}
  If follows that with probability greater than $1-(p_B+p_C+p_B)$
  \[
    \mathcal E[\wh{P}_{\wh{\Lm}_{\eta}}] \lesssim
    \underbrace{\eta^{\frac{4\sigma}{2\sigma+1}}}_{A} +
    \underbrace{\eta^{\frac{4\sigma}{2\sigma+1}} t}_{C} 
  \]
\textit{ i.e.}
  \[
    \PP{\mathcal E[\wh{P}_{\wh{\Lm}_{\eta}}] \gtrsim
      \eta^{\frac{4\sigma}{2\sigma+1}} (1+t) } \lesssim
   \underbrace{ (n^\gamma+\eta^{-\frac{2}{2\sigma+1}}) \exp\left( -c_a n\eta^2 t
    \right)}_{p_A+p_D} +  \underbrace{\eta^{-\frac{2}{2\sigma+1}}
      \exp\left( -c_{\MT} n \eta^2 t\right) }_{p_C}
  \]
  which gives~\eqref{eq:50}.
\end{proof}

\begin{proof}[Proof of Cor.~$\ref{pasquale}$]
 Since $\eta_n^2= \frac{(\gamma+\beta)\ln n}{c_an}$, then 
bound~\eqref{eq:50} gives~\eqref{eq:56} since
  \begin{alignat*}{1}
    & (n^\gamma+ \left(\frac{c_an}{(\gamma+\beta)\ln n}\right)^{\frac{1}{2\sigma+1}}) \exp\left( -c_a n\eta_n^2
    \right) \lesssim n^{\gamma}  n^{-(\gamma+\beta)} =n^{-\beta}\\
& \left( \frac{c_an}{(\gamma+\beta)\ln n}    \right)^{\frac{1}{2\sigma+1}}  \exp\left( -c_\MT n\eta_n^2 t
    \right) \lesssim n  n^{- \overline{c}_\MT  t} = 
    n^{1-\overline{c}_\MT t} 
  \end{alignat*}
where $\overline{c}_\MT= c_\MT c_a^{-1} (\gamma+\beta)$.
Eq.~\eqref{eq:69} is clear. 
\end{proof}

\begin{proof}[Proof of Prop.~$\ref{apprx}$]
Given $I\in\MT$, by~\eqref{eq:16} and~\eqref{eq:6}, 
  \[
\sum_{J\in\mathcal C^{N+1}(I)} {\mathcal E}_J \lesssim a\, b^{-2(j_I+N+1)} 
    \]
so that   
\[
\lim_{N\to+\infty } \sum_{J\in\mathcal C^{N+1}(I)} {\mathcal E}_J =0
\]
 and, by taking the limit in~\eqref{eq:18}, 
\[
{\mathcal E}_{I} = \sum_{k=0}^{+\infty} \sum_{J\in\mathcal C^k(I)} \epsilon_J^2.
\]
Set $\T_k= \T_{\eta/2^k}$ for all $k\in\N$, then
\begin{alignat*}{1}
   \mathcal E(P_{\Lm_{\eta}})  & = \sum_{I\in \Lm_{\eta}} {\mathcal E}_I = \sum_{I\in \Lm_{\eta}} \sum_{k=0}^{+\infty}
   \sum_{J\in\mathcal C^k(I)} \epsilon_J^2 = \sum_{J\notin\T_\eta} \epsilon_J^2 
   =\sum_{k=0}^{+\infty}\quad \sum_{J\in \T_{k+1}\setminus \T_k}
  \epsilon_J^2 \leq \sum_{k=0}^{+\infty} \sharp \T_{k+1}\,
  (\frac{\eta}{2^k})^2  \\&
  \lesssim  \sum_{k=0}^{+\infty}  \frac{\eta^2}{2^{2k}} (\frac{\eta}{2^{k+1}})^{-\frac{2}{2\sigma+1}}
 = \eta^{\frac{4\sigma}{2\sigma +1}} \sum_{k=0}^{+\infty} 4^{\frac{k+1}{2\sigma+1}-k}\lesssim \eta^{\frac{4\sigma}{2\sigma+1}} ,
\end{alignat*}
where the first inequality  is a consequence of the fact that $\eps_J<
(\frac{\eta}{2^k})^2$ if $J\notin \T_k$, the second inequality follows
from~\eqref{eq:13bis}, whereas the last inequality
holds since the  series $\sum_{k=0}^{+\infty} 4^{-\frac{2\sigma
    k-1}{2\sigma+1}}$  converges.
\end{proof}

\section{Conclusions}\label{sec:conc}
In this paper, we proposed and analyzed a multiscale vector quantization approach inspired by ideas in   \cite{lima16}. We provided non asymptotic error bounds on the corresponding distortion/reconstruction error combining geometric and probabilistic tools. The analysis is developed in a general setting with manifold supported data as a special case. 

A number of research directions remain to be explored. Following again  \cite{lima16}, it would be interesting to understand the role of noise and the impact of data driven partitioning. Further,  following the parallel with wavelets it would be interesting to use results from reproducing kernel Hilbert spaces to develop Gabor like geometric wavelets and analyse the properties of the corresponding quantization schemes. 

\section*{Acknowledgments}
 L. R. acknowledges the financial support of the AFOSR projects FA9550-17-1-0390 and BAA-AFRL-AFOSR-2016-0007 
(European Office of Aerospace Research and Development), and the EU H2020-MSCA-RISE project NoMADS - DLV-777826.

\appendix

\section{Further Proofs}\label{sec:proofs}

This section is devoted to show all the results cited in the proof in
Section~\ref{mainproof}. 
We first show some preliminary results. 

\subsection{Technical results}
We introduce some further notations about partition trees.

With slight abuse of notation, we regard ${\MT =\cup_{j\in\N} \cup_{I\in \Lambda_j} I}$ as the (disjoint) union
of the cells in each partition~$\Lambda_j$. 
It a cell does not
split, \textit{ i.e.} $\mathcal  C(I)= I$, we regard
$I\in\Lambda_j$ and   $\mathcal C(I)\in  \Lambda_{j+1}$ as different cells.  

Given a cell $I\in\MT$, for any $k\in\N$ we set
\[ \mathcal C^{k+1}(I) =\mathcal C (\mathcal C^{k}(I)) \qquad \mathcal
  P^{k+1}(I) =\mathcal P (\mathcal P^{k}(I)), \]
where $\mathcal C^0(I)=\mathcal P^0(I)=\set{I}$ and, clearly,
\[ \mathcal P^{k}(I)=\set{\X} \qquad k\geq j_I.\]
Furthermore, for any pair $I,J\in\MT$, $I\neq J$ one and only one of the following
alternative possibilities holds true
\begin{align}
  \label{eq:26}
  I\cap J= \emptyset \qquad \text{ or } \qquad  J\in\mathcal C^k(I) \qquad
  \text{ or } \qquad I\in\mathcal C^k(J)
\end{align}
for some $k\geq 1$.

If $\{\mathcal T_t\}_{t\in T}$ is an arbitrary
family of  subtrees, clearly the intersection $\cap_{t\in T} \mathcal T_t$
and the union $\cup_{t} \mathcal T_t$ are the \textit{ smallest } 
and the \textit{ largest}  subtrees in the family.

Given a subset $S\subset\MT$, we set
\[ \T(S)=\bigcap\set{ \T \mid \T \text{ is a subtree and } S\subset\T},\]
which  is the smallest subtree containing all the cells in $S$, and
\begin{subequations}
  \begin{alignat}{1}
    \T(S) & =\bigcup_{I\in S} \set{\mathcal P^k(I)\mid k=0,\ldots,j_I} =
    \bigcup_{I\in S} \T(I) \label{eq:26a} \\
   \sharp\T(S) & \leq 1+\sum_{I\in S}  {j_I} \label{eq:26b}.
  \end{alignat}
\end{subequations}

The following remark provides an alternative definition of
$\hat{\T}_\eta$ and
a similar procedure can be applied to $\T_\eta$
\begin{remark}\label{alternative} 
  Given $\eta>0$ and $j_n\in\N$, set
  \[
\hat{S}_\eta = \{I\in\MT \mid j_I\geq j_n , \hat{\eps}_I\geq \eta\}
\]
which is a finite subset of $\MT$. It is easy to check that $\mathcal
T(\hat{S}_\eta) = \hat{T}_\eta$.
\end{remark}


The following lemma is quite obvious. 
\begin{lemma}\label{lem:wba}
If $\mathcal T$ is finite subtree, the family of
outer leaves
\[
\Lm(\mathcal T)=\{ I\in\mathbb T : I\not\in \mathcal T, \mathcal P(I)\in\mathcal T\}
\]
is a partition with
\begin{align}
\sharp \Lm(\mathcal T)\leq  (a-1)\, \sharp \mathcal T+1 \leq a\, \sharp \mathcal
T\label{eq:cardinality} .
\end{align}
\end{lemma}
 
\begin{proof}
  For any $j\in\N$, we fix an arbitrary order among the cells in $\Lambda_j$, \textit{ i.e.}
  \[
\Lambda_j= \set{I_{j,1},\ldots, I_{j,N_j}} \qquad N_j=\#\Lambda_j .
    \]
Given two different cells $I=I_{j,k}$ and $ J=I_{j',k'}$ in $\MT$, we
say that  $I$ is older than $J$  either 
if $j<j'$ or $j=j'$ and $k<k'$.

By definition $\T$ contains the parents of all its elements
and, hence,  the root  
$\X$.  Define,  by induction,  the  family of  subtrees 
\[  \T_1=\{\X\}\qquad  \T_{n+1}=\T_n\cup \set{I},\] 
 where  $I$ is the oldest cell in $\T\setminus \T_n$. Note that,  by
 construction,  $I\in \Lm(\T_n)$.  Since
 $\T$ is finite by assumption,  $\T_{N}=\T$ with $N=\sharp \T$.

We now prove by induction on $n=1,\ldots,N$ that $\Lambda(\mathcal
T_n)$ is a partition. If $n=1$
\[ \Lambda(\T_1)=\set{I\in\MT | I\neq\X,\mathcal P(I)=\X}=\mathcal C(\X)= \Lm_1\]
which is a partition by assumption and, by~\eqref{eq:22}, it satisfies
\eqref{eq:cardinality}. 
Assume that  $\Lm(\T_n)$ is  
partition satisfying \eqref{eq:cardinality} with $\T=\T_n$. By
construction, $\T_{n+1}=\T_n\cup I$  with  $I\in \Lm(\T_n)$, then 
\[ \Lm(\T_{n+1})=\left(\bigcup_{J\in C(I)} \set{J} \right)\bigcup (\Lm(\T_{n})\setminus \set{I})\]
which is a partition since $I= \cup_{J\in C(I)} J$ and
\[ \sharp\Lm(\T_{n+1})\leq  a + \sharp \Lm(\T_{n})-1 \leq a+ (a-1) \sharp
\T_{n} = (a-1) (\sharp \T_{n}+1) +1=(a-1) (\sharp \T_{n+1}) +1,\]
so that \eqref{eq:cardinality} holds true with $\T=\T_{n+1}$.
\end{proof}

We observe that, given a cell $I$, by definition of~$\diam{I}$,  it exists a measurable subset $J\subset I$ such that $\rho(J)=0$ and $\operatorname{diam}(I\setminus
J)= \diam{I} $. Furthermore,
\[\PP{ \norm{X_i-X_j}>\diam{I}| X_i,X_j\in I} = 0 \qquad i,j=1,\ldots,n.\] 
The following simple lemmas will be useful. 
\begin{lemma}\label{lem:diam}
Given a cell $I\in\MT$ with $\rho(I)>0$
\begin{subequations}
  \begin{alignat}{1}
    \norm{x-c_I} & \leq
    \diam{I} \qquad \rho-\text{almost all } x\in I\label{eq:63}\\
\norm{ c_J-c_K} & \leq \diam{I}  \label{eq:58} \qquad J,K\in \mathcal C(I)\\
    \norm{\wh{c}_I-c_I} & \leq
    \begin{cases}
      \diam{I}  & n_I \neq 0 \\
      \operatorname{diam(\X)} &  n_I = 0 
    \end{cases}
    \quad \text{ almost surely}
\label{eq:54} 
  \end{alignat}
\end{subequations}
\end{lemma}
\begin{proof}
  By definition of essential diameter, there exists
  $I_0\subset I$ such that $\operatorname{diam}(I_0)=\diam{I}$
  and  $\rho(I\setminus I_0)=0$. Let $C$ the closed convex hull of $I_0$.  
  It is known that $\operatorname{diam}(C)= \diam{I}$ and, 
  by convexity theorem, see \cite[Thm. 5.7.35]{schwartz},
\[
c_I=\frac{1}{\rho(I_0)} \int\limits_{I_0} x\,d\rho(x) \in C,
\]
so that~\eqref{eq:63} is clear.  Since $J,K\subset I$,
Eq.~\eqref{eq:58} is a consequence of the fact that  $c_J,c_K\in C$.
If $n_I=0$, $\wh{c}_I=\hat{x}_I^*\in\X$ so that~\eqref{eq:54} is
clear. If $n_I\neq 0$,  almost surely $\wh{c}_I\in C$ so that 
\[   \norm{\wh{c}_I-c_I} \leq  
  \operatorname{diam}(C)=\operatorname{diam}(I_0)=\diam{I}.\] 
\end{proof}

Given a cell $I\in\MT$,  the within-between decomposition of
the variance 
\begin{align}
  \label{eq:23}
  {\mathcal E}_I =  \sum_{J\in\mathcal C(I)} {\mathcal E}_J +  \sum_{J\in\mathcal C(I)} \rho_I \norm{c_J-c_I}^2,
\end{align}
implies
\begin{alignat}{1}
  \label{eq:22a}
   \eps_I^2& = \sum_{J\in\mathcal C(I)} \rho_I \norm{c_J-c_I}^2 
\end{alignat}
As a consequence we have the following decomposition.
\begin{lemma}\label{lem:series}
Given $I\in\MT$, for all $N\in\N$
\begin{align}
  \label{eq:18}
 {\mathcal E}_I= \sum_{k=0}^N \sum_{J\in\mathcal C^k(I)} \epsilon_J^2 +   
\sum_{J\in\mathcal C^{N+1}(I)} {\mathcal E}_J.
\end{align}
\end{lemma}
\begin{proof}
The claim is clear for $N=0$. Assume that it holds true for $N$. Then, 
for any $J\in  \mathcal C^{N+1}(I)$
\[ {\mathcal E}_J= \epsilon_J^2 + \sum_{J'\in \mathcal C(J)} {\mathcal E}_{J'},\]
hence
\[ {\mathcal E}_I= \sum_{k=0}^N \sum_{J\in\mathcal C^k(I)} \epsilon_J^2 +
  \sum_{J\in\mathcal C^{N+1}(I)} \left( \epsilon_{J}^2 + \sum_{J'\in
      \mathcal C(J)} {\mathcal E}_{J'} \right),\]
by observing that a cell $J'\in \mathcal C(J)$ for some $J\in\mathcal C^{N+1}(I)$  if
 and only if $J'\in \mathcal C^{N+2}(I)$, so that~\eqref{eq:18} holds true for $N+1$.
\end{proof}

We now show that the set $\T_\eta$ defined by~\eqref{eq:46tris} is a
finite subtree. 
\begin{lemma}\label{lem:finite}
The family  $\T_\eta$ is a finite subtree of $\MT$.
\end{lemma}
\begin{proof}
If $\T_\eta=\set{\X}$, there is nothing to prove.  Otherwise, if $I\in
\T_\eta$, then by definition there exists $J\in \mathbb T$
such that $ J\subset I$ and $\eps_J\geq \eta$. Since $\mathcal
P(I)\supset I\supset J$, then $P(I)\in\T_\eta$, so that  $\T_\eta$ is
a subtree.

We now show that $\mathcal T_\eta$ is finite.  From~\eqref{eq:18} with
$I=\X$ and~\eqref{eq:16}
we get that for all $N\in\N$
\[\sum_{k=0}^N \sum_{J\in\mathcal C^k(\X)} \epsilon_J^2\leq   \ {\mathcal E}_{\X}
  \leq 1.\]
Then the series $\sum_{k=0}^{+\infty} \sum_{J\in\mathcal C^k(\X)}
\epsilon_J^2=\sum_{I\in\MT} \epsilon_J^2$ is sommable. Hence, the
set
\[
S_\eta=\{I\in\MT :  \epsilon_I\geq\eta,\}
\]
is finite. Furthermore, by construction
\begin{align}
  \label{eq:25}
  \T_\eta =\T(S_\eta)= \bigcup_{I\in S_\eta} \T(I).
\end{align}
Bound~\eqref{eq:26b} implies that $\T_\eta$ is finite.
\end{proof}

\begin{remark}\label{eta}
By~\eqref{eq:11b},
  \begin{align}
    \label{eq:16}
    \epsilon_I^2\leq \Er_I \leq \diam{I}^2 \rho_I \leq \operatorname{diam}(\X)^2\leq 1,
  \end{align}
  and $\eps_I< \Er_I\leq 1$ provided that $\eps_I>0$.  Furthermore,
  by~\eqref{eq:16} $\T_\eta = \Lambda_{\eta}=\set{\X}$ for all $\eta\geq 1$ and
  $\mathcal E( P_{\Lambda_{\eta}})=\Er_\X\leq 1$.  By the same argument
  $\hat{\T}_\eta = \hat{\Lambda}_{\eta}=\set{\X}$.  Hence,   it is
  enough to consider the case $0<\eta<1$. 
\end{remark}

\subsection{$A$-term: Approximation error}
In this section, we bound the approximation error, which is based in
an estimation of the number of cells $I$ such that $\eps_I$ is  big
enough. 
\begin{lemma}
Given a partition $\Lm \subset \MT$,  given $0<\eta<1$,
\begin{align}
  \label{eq:2}
 \sharp \{ I \in\Lm : \eps_I\geq \eta\} \leq \sharp \{ I \in\Lm : {\mathcal E}_I\geq \eta^2\} \lesssim   \eta^{-\frac{2}{2s+1}}.
\end{align}
\end{lemma}
\begin{proof}
First inequality in~\eqref{eq:2} is a direct consequence
of~\eqref{eq:16}.  By~\eqref{eq:1} there exists $C>0$ such that for
all $I\in\MT$ 
\[ \diam{I} \leq C \rho_I^s\qquad I\in\Lm,\]
then, by~\eqref{eq:16}
\[ {\mathcal E}_I \leq C^2 \rho_I^{2s+1}.\]
Set $\Lm_+=   \{ I \in\Lm : {\mathcal E}_I\geq \eta^2\}$ and $N_+= \sharp \Lm_+$.
Fix $q\geq 1$, clearly
\begin{align}
\eta^{\frac{2}{q}} N_+ \leq \sum_{I\in\Lm_+} {\mathcal E}_I^{\frac{1}{q}}\leq
C^{\frac{2}{q}} \sum_{I\in\Lm_+}\rho_I^{\frac{2s+1}{q}}  
\label{eq:3}.
\end{align}
The Holder inequality with $1/p+1/q=1$ gives
\begin{alignat*}{1}
  \sum_{I\in\Lm_+}  \rho_I^{\frac{2s+1}{q}} & \leq \left(\sum_{I\in\Lm_+}
  \rho_I^{\frac{p(2s+1)}{q}} \right)^{\frac{1}{p}} 
  \left(\sum_{I\in\Lm_+} 1^q\right)^{\frac{1}{q}}\leq \left(\sum_{I\in\Lm} \rho_I \right)^{\frac{2s+1}{2s+2}}
 N_+^{\frac{1}{2s+2}} =  N_+^{\frac{1}{2s+2}} ,
\end{alignat*}
where the last inequality follows by choosing $\frac{p(2s+1)}{q}=1$,
\textit{ i.e.} $q=2s+2$.  By replacing in~\eqref{eq:3}
\[ N_+^{1-\frac{1}{2s+2}} \leq   C^{\frac{1}{s+1}}
  \eta^{-\frac{2}{2s+2}} 
\]
and, since $1-\frac{1}{2s+2}=\frac{2s+1}{2s+2}$, we get~\eqref{eq:2}.
\end{proof}
We first observe that the proof of the above lemma only depends on
Assumption~\eqref{eq:1} and the constant in the
inequality~\eqref{eq:2} only depends on the constant in~\eqref{eq:1},
denoted by $C$ in the proof.    
Furthermore, without
Assumption~\eqref{eq:1} we always have the following bound
\[
\sharp \{ I \in\Lm : {\mathcal E}_I\geq
\eta^2\}\leq \sharp \{ I \in\Lm : {\mathcal E}_I\geq
\eta^2\}  \leq \eta^{-2},
\]   
where the first inequality is consequence of~\eqref{eq:16}  and the
last bounds is due to~\eqref{eq:23} with $I=\X$ 
 \begin{alignat*}{1}
 1\geq {\mathcal E}_{\X} \geq \sum_{I\in\Lm} {\mathcal E}_I \geq
 \sum_{I\in\Lm,{\mathcal E}_I\geq\eta^2} {\mathcal E}_I = \eta^2\,\sharp \{ I \in\Lm : {\mathcal E}_I\geq  \eta^2\}.
 \end{alignat*} 
We recall that  $\Lm_\eta=\Lm( T_\eta)$ is the family of  the
corresponding outer leaves, which is a partition of $\X$ by
Lemma~\ref{lem:wba}. 
In order to bound the cardinality of $\Lm_\eta $ we need an auxiliary
lemma based on Assumption~\eqref{eq:6}. 
\begin{lemma}\label{max_depth}
Given $\eta>0$, set
  \begin{align}
j_\eta = \sup\{ j_I\in\N \mid I\in \MT  \text{  and } \eps_I \geq \eta\},\label{eq:17}
\end{align}
then
\begin{align}
  \label{eq:4}
j_\eta \lesssim \ln
  (\frac{2}{\eta}).
\end{align}
\end{lemma}
If $\{ j_I\in\N \mid I\in \MT  \text{  and } \eps_I \geq
\eta\}=\emptyset$ we set $j_\eta=0$.

\begin{proof}
If $\T_\eta=\set{\X}$ or $\{ j_I\in\N \mid I\in \MT  \text{  and } \eps_I \geq
\eta\}=\emptyset$, then  $j_\eta=0$, so that the claim
is evident. If  $\T_\eta\neq\set{\X}$, then $0<\eta< 1$.
Take  $I\in \MT$ such that $\eps_I \geq\eta$. By~\eqref{eq:16}
and~\eqref{eq:6} 
\begin{alignat*}{1}
  \eta^2 &\leq \eps_I^2 \leq \rho_I \diam{I}^2  \leq C_2  b ^{-2 j_I}.
\end{alignat*}
Hence
\[ j_I \leq \frac{1}{\ln b} \ln (\frac{1}{\eta}) +\frac{1}{2\ln b}
  \ln C_2 \leq E \ln (\frac{2}{\eta}) ,\] 
where $E=\max\set{1,\frac{\ln C_2}{2\ln 2} }/ \ln b$.
\end{proof}

\begin{proposition}
Given $\eta>0$, for all $\sigma <s$,
\begin{subequations}
  \begin{alignat}{1}
\sharp\T_\eta & \lesssim \,  \eta^{-\frac{2}{2s+1}}
    \ln(\frac{2}{\eta}) \lesssim \,  \eta^{-\frac{2}{2\sigma+1}}\label{eq:13bis} \\
    \sharp\Lm_\eta & \lesssim \,  \eta^{-\frac{2}{2s+1}}\ln(\frac{2}{\eta}) \lesssim \,  \eta^{-\frac{2}{2\sigma+1}}
    \label{eq:13},
  \end{alignat}
\end{subequations}
where the constants in $\lesssim$ also depend on $\sigma$.
\end{proposition}

\begin{proof}
As observed in Remark~\ref{eta}, we can assume that $0<\eta<1$.   
Let 
\[\Upsilon=\set{ I\in\T_\eta\mid \epsilon_I\geq\eta \text{ and }\epsilon_J<\eta \,
    \forall J\in\mathcal C^k(I),\ k\geq 1  }.\]  
By~\eqref{eq:26}  the elements of $\Upsilon$ are disjoint. Let
$\Lm\subset\MT$ be a partition such that $\Upsilon\subset \Lm$. Hence,
by~\eqref{eq:2} 
\begin{align}
 \sharp\Upsilon \leq \sharp\set{I\in\Lm\mid \eps_I\geq \eta }\lesssim \eta^{-\frac{2}{2s+1}}\label{eq:28}.
\end{align}
We claim that
\[ \T_\eta = \bigcup_{ J\in \Upsilon} \mathcal \T(J)=\T(\Upsilon).\]
By construction $\T_\eta \supset \bigcup_{ J\in \Upsilon} \mathcal
\T(J)$. To prove the opposite inclusion, fix $I\in \T_\eta$, then 
there exists  $J_1\in\MT$ such
that $J_1\subset I$ and $\epsilon_{J_1}\geq \eta$. If $J_1\in\Upsilon$,
then $I\in\T(J_1)$. Otherwise, there exists $J_2\subset J_1\subset I$
and $\epsilon_{J_2}\geq \eta$. If $J_2 \in\Upsilon$, $I\in\T(J_2)$. 
Otherwise, because of $\T_\eta$ is finite,  we can repeat the procedure until we
get  $J_k \in\Upsilon$ 
such that $J_k\subset I$, then $I\in\T(J_k)$ and the claim is
proven. By~\eqref{eq:26b}
\[ \sharp \T_\eta =\sharp \T(\Upsilon)\leq1+ \sum _{ J\in \Upsilon} j_J 
 \leq 1+\sharp \Upsilon\  j_\eta \lesssim  \eta^{-\frac{2}{2s+1}} \ln(\frac{2}{\eta}),\] 
since, by definition, $j_J\leq j_\eta$ and the last inequality is due to~\eqref{eq:28} and~\eqref{eq:4}.
This shows  the first inequality in~\eqref{eq:13bis}.  Since
 $\sigma<s$, for some $\delta>0$
\[
\eta^{-\frac{2}{2s+1}} \ln(\frac{2}{\eta}) = \eta^{-\frac{2}{2\sigma
    +1}} \eta^{\delta}\ln(\frac{2}{\eta}) \leq C \eta^{-\frac{2}{2\sigma
    +1}}
\]
where $C=\sup_{ 0<\eta\leq 1} \eta^\delta \ln(\frac{2}{\eta})$, which is
finite, since $\lim_{\eta\to 0} \eta^\delta \ln(\frac{2}{\eta})=0$.
Bound~\eqref{eq:13} is a direct consequence of~\eqref{eq:cardinality}.
\end{proof}
\begin{remark}\label{log}
In the following, for sake of clarity we bound the logarithmic
dependence on $\eta$ by considering $\sigma<s$. However our results
can be extended to $\sigma=s$ by adding a logarithmic factor, as
in~\eqref{eq:13bis} and~\eqref{eq:13}.  
\end{remark}

\subsection{$C$-term: sample error}
The following result bounds the sample error for a given partition.
\begin{proposition}\label{prop:3}
Fix a data-independent subset $\mathcal I\subset\MT$ of cells and
$N>0$. Given a partition
$\wh{\Lm}\subset\mathcal I$ (possibly depending on the data) such that
$\sharp \wh{\Lm}\leq N$, for any $t>0$,
\begin{align}
  \label{eq:20}
 \PP{\int_{\X} \norm{P_{\wh{\Lm} }(x) -  \wh{P}_{\wh{\Lm}
   }(x)}^2\,d\rho(x) > t } \leq 2\,\sharp\mathcal I \exp\left( -
  \frac{nt }{8N} \right) .
\end{align}
\end{proposition}
\begin{proof} 
Consider the following event
\[ \Omega =\bigcup_{I\in \mathcal I}  \set{ \sqrt{\rho_I}\norm{ \wh{c}_I-c_I}> t},\]
which is well-defined since $\mathcal I$ does not depend on the data
$X_1,\ldots,X_n$. By  union bound 
   \begin{align}
\Prob[ \Omega] \leq \sharp\mathcal I \sup_{\substack{I\in \mathcal I\\
    \rho_I>0}} \Prob[ 
 \norm{\wh{c}_I- c_I}> \frac{t}{\sqrt{\rho_I}}].\label{eq:21}
\end{align}
Fix $I\in\mathcal I$ with $\rho_I >0$. The tower property with
respect to the binomial random variable $n_I$ gives
\begin{alignat*}{1}
  \Prob[\norm{\wh{c}_I- c_I}> t]& =\sum_{k=0}^n \binom{n}{k}
  \rho_I^k (1-\rho_I)^{n-k} \,\Prob[\norm{\wh{c}_I- c_I}>t \mid n_I=k].
\end{alignat*}
Conditionally to the event $\set{n_I=k}$ with $k>0$, up to a
permutation of the indexes, we can assume that $X_1,\ldots,X_k\in I$
and $X_{k+1},\ldots,X_n\notin I$. Furthermore, 
\[ \wh{c}_I- c_I= \frac{1}{k} \sum_{i=1}^k (X_i - c_I) =  \frac{1}{k} \sum_{i=1}^k \xi_{i}\]
where $ \xi_{1},\ldots, \xi_{k}$ are independent zero mean
random vectors  bounded by $M=\diam{I}$ almost surely by~\eqref{eq:63}. 
Hence, by~\eqref{92b}
\[ \Prob[\norm{ \wh{c}_I- c_I}>t \mid n_I=k] \leq 2 \exp(-
  \frac{kt^2}{4\diam{I}^2} ) ,\]
which trivially holds true also if $k=0$.
Hence,
\begin{alignat*}{1}
  \Prob[\norm{\wh{c}_I- c_I}> t]& \leq 2 \sum_{k=0}^n \binom{n}{k}
  ( \rho_I \exp(-\frac{t^2}{4\diam{I}^2} ) )^k (1-\rho_I)^{n-k}   
  =  2 \left( 1 -\rho_I \left(1- \exp(-\frac{t^2}{4\diam{I}^2})\right)
 \right)^n  \\ &
 \leq 2 \exp\left( -n\rho_I \left(1- \exp(-\frac{t^2}{4\diam{I}^2})\right)
 \right) 
\end{alignat*}
where in the fourth line we use the bound $(1-\tau)^n\leq \exp(-n\tau)$ with
$0\leq \tau \leq 1$.  Since
\[ 1-\exp(-\tau)\geq \frac{\tau}{2}  \qquad \text{for all }\tau\leq 1,\] 
then,  for all $t\leq \diam{I}$,
  \begin{align}
\Prob[\norm{\wh{c}_I- c_I}> t] \leq 2 \exp\left( -n\rho_I
    \frac{t^2}{8\diam{I}^2} \right) \leq 2 \exp\left( -n\rho_I
    \frac{t^2}{8} \right).\label{eq:61} 
\end{align}
If $\diam{I}<t\leq \operatorname{diam}{\X}\leq 1$, by~\eqref{eq:54},
clearly $\Prob[\norm{\wh{c}_I- c_I}>  t\mid n_I>0] =0$, so that
\[\Prob[\norm{\wh{c}_I- c_I}> t]= \Prob[\norm{\wh{x}^*_I- c_I}> t]\mid
  n_I=0] \PP{n_I=0} \leq  (1-\rho_I)^n \leq   \exp(-n\rho_I) \leq 2\exp\left( -n\rho_I
    \frac{t^2}{8} \right),
\]
where the last bound holds true for any $t\leq 2\sqrt{2}$. 
Finally, if $t>\operatorname{diam}{\X}$,  as above
\[
\Prob[\norm{\wh{c}_I- c_I}> t]= \Prob[\norm{\wh{x}^*_I- c_I}> t]\mid
  n_I=0] \PP{n_I=0} =0
  \]
since $\wh{x}^*_I, c_I\in\X$, compare with~\eqref{eq:54}.
It follows that~\eqref{eq:61} holds
  true for all $t>0$. Summarizing, from~\eqref{eq:21} we get
\[
\Prob[ \Omega] \leq  2\,\sharp\mathcal I \exp\left( -\frac{nt^2}{8}
  \right).
\] 
Since $\wh{\Lm}\subset\mathcal I$,  on the complement of $\Omega$,
\begin{alignat*}{1}
   \int_{\X} \norm{P_{\wh{\Lm} }(x) -  \wh{P}_{\wh{\Lm}
     }(x)}^2\,d\rho(x)  =\sum_{\substack{I\in\wh{\Lm}\\\rho_I>0}} \rho_I \norm{\wh{c}_I- c_I
     }^2  \leq N t^2,
\end{alignat*}
and bound~\eqref{eq:20} follows by replacing $t$ with $\sqrt{t/N}$.
\end{proof}

\begin{remark}\label{rem:choice}
By inspecting the proof, it is possible to check that the assumption
$\hat{x_I}^*\in\X$ is needed only in this proposition and it can be
replaced by the condition that $\inf_{x\in\X} \norm{\hat{x}_I^*- x} \leq 1 $, so that
$\norm{\hat{x}_I^*- c_I} \leq 2$.  
 \end{remark}

\subsection{$B$ and $D$ terms: Stability of $P_{\Lm}$ with respect to
  the partition.} 

The following result is well-known.

\begin{lemma}
 Given a cell $I\in\MT$ with $\rho_I>0$, for all $t>0$
 \begin{align}
   \label{eq:41}
   \PP{ \abs{ \frac{n_I}{n} - \rho_I}\geq \rho_I t} \leq 2
   \exp\left(-\frac{n\rho_It^2}{2(1+t/3)}\right) \leq 2
   \exp\left(-\frac{n\rho_It^2}{M_I}\right)  ,
 \end{align}
where $M_I=2/3\max\set{4,(1+2\rho_I)/\rho_1}$.
\end{lemma}
\begin{proof}
We apply the Bernstein inequality, see 
\cite[Cor. 2.11]{boluma13}, to the family of independent random variables
$\mathbbm 1_I(X_1),\ldots, \mathbbm 1_I(X_n)$, which satisfy
\begin{alignat*}{2}
  \mathbb E[\mathbbm 1_I(X_i)] & = \rho_I & \quad i=1,\ldots,n\\
  \sum_{i=1}^n \mathbb E[\mathbbm 1_I(X_i)^2] & = n \rho_I & \\
  \sum_{i=1}^n \mathbb E[\mathbbm 1_I(X_i)^m] & = n \rho_I \leq \frac{n \rho_I}{2}
  m! (\frac{1}{3})^{m-2} & \quad m\in\N,\, m\geq 3,
\end{alignat*}
then, 
\[
\PP{ \abs{ n_I - n\rho_I}\geq n\rho_I t} \leq 2 
\exp\left(-\frac{(n\rho_It)^2}{ 2(n\rho_I + n\rho_I t/3)}\right) = 2
   \exp\left(-\frac{n\rho_It^2}{2(1+t/3)}\right).
 \]
 Observing that
 \[ \abs{ n_I - n\rho_I}\leq n\max\set{\rho_I,1-\rho_I},\]
if $t>\max\set{1,1/\rho_I-1}=t^*$, then $\PP{ \abs{ n_I -
      n\rho_I}\geq n\rho_I t}=0$, whereas  if $t\leq \max\set{1,1/\rho_I-1}=t^*$,
it holds that
\[ 2(1+t/3)\leq 2(1+t^*/3)=M_I,\]
so that the second bound in~\eqref{eq:41} is clear. 
\end{proof}

The following lemma provides a concentration inequality of
$\sqrt{\frac{n_I}{n}}$.
\begin{lemma}
 Given a cell $I\in\MT$ with $\rho_I>0$, for all $t>0$
 \begin{align}
   \label{eq:41maurer}
   \PP{ \abs{ \sqrt{\frac{n_I}{n}} - \sqrt{\rho_I}}\geq t} \leq 2
   \exp\left(-\frac{n t^2}{2}\right).
 \end{align}
\end{lemma}
\begin{proof}
  We apply Proposition~\ref{maurer}  with $\mathcal Y=\set{0,1}$
  \[
\xi_i=\mathbbm 1_I(X_i)  \qquad
f(y_1,\ldots,y_n)=\frac{1}{n}\sum_{i=1}^n y_i,
\]
where $f$ is clearly bounded, and 
\[
f(\xi_1,\ldots,\xi_n) =\frac{n_I}{n} \qquad \EE{f(\xi_1,\ldots,\xi_n) } =\rho_I.
  \] 
Given $k=1,\ldots,n$, it holds that
\[
V_k(\xi_1,\ldots,\xi_n)=\frac{1}{n} \sup_{y\in\mathcal Y} \left(\mathbbm 1_I(X_i) -y\right)
= \frac{1}{n}  \mathbbm 1_I(X_i)\leq \frac{1}{n} ,\]
so that $\alpha=1/n$. Furthermore
\[
\sum_{k} V^2_k(\xi_1,\ldots,\xi_n) = \frac{1}{n^2} \sum_{k}  \mathbbm
1_I(X_i) = \frac{1}{n} f(\xi_1,\ldots,\xi_n),
 \]
then $\beta=1/n$ and Eq.~\eqref{eq:29}  implies~\eqref{eq:41maurer}. 
\end{proof}

The following lemma shows that, given a cell $I\in\MT$, $\wh{\eps}_I$
concentrates around $\eps_I$. 
\begin{lemma}
Given $I\in \MT$, for all $t>0$
\begin{align}
  \label{eq:64}
  \Prob\left[  \lvert \wh{\eps}_I- \eps_I\rvert >t \right] \leq 2 \ell
  \exp\left(- n \frac{t^2}{64\ell\diam{I}^2}\right).
\end{align}
where $\ell=1+\sharp \mathcal C(I)$.
\end{lemma}
\begin{proof}
Fix $I \in\MT$.  If $\rho_J=0$ for some $J\in\mathcal C(I)$, then almost surely
$X_i\notin J$ for all $i=1,\ldots,n$ and, hence,  $n_J=0$, so that  
both $\wh{\eps}_I$ and $\eps_I$ do not depend on the children
$J$. Hence,  without loss of generality, we can assume that $\rho_J>0$ for all
$J\in\mathcal C(I)$. 

Let $\ell=\sharp \mathcal C(I)+1$.  Set $L^2(\mathcal C(I))=\R^{\ell-1}$,
regarded as Euclidean vector space whose norm is denoted by
$\norm{\cdot}_2$.  Define
$v,\wh{v},\wh{w}\in L^2(\mathcal C(I))$ 
\begin{alignat*}{1}
  v(J)=& \sqrt{\rho_J} \norm{c_J-c_I}  \qquad 
  \wh{v}(J)= \sqrt{\frac{n_J}{n}} \norm{\wh{c}_J-\wh{c}_I} \qquad 
  \wh{w}(J)=  \sqrt{\rho_J} \norm{\wh{c}_J-\wh{c}_I}.
\end{alignat*}
Then 
\begin{alignat}{1}
  \lvert  \wh{\eps}_I- \eps_I\rvert & = \lvert  \norm{\wh{v}}_2-
  \norm{v}_2 \rvert \leq \norm{\wh{v}-v}_2  \leq
  \norm{\wh{w}-v}_2 + \norm{\wh{v}-\wh{w}}_2 . \label{eq:66}
\end{alignat}
We now bound the first term. Set $\mathcal I=\mathcal C(I) + \{I\}$
and $\sharp \mathcal  I=\ell$.   Then
\begin{alignat*}{1}
\norm{\wh{w}-v}_2^2  & = \sum_{J\in\mathcal C(I)} \rho_J
\left| \norm{\wh{c}_J-\wh{c}_I} -\norm{c_J-c_I}  \right|^2 
\leq 2 \sum_{J\in\mathcal C(I)} \rho_J
 \left( \norm{\wh{c}_J-c_J}^2 +\norm{ \wh{c}_I-c_I}^2  \right)  \\&
 \leq 2 \sum_{J\in \mathcal I} \rho_J
 \norm{\wh{c}_J-c_J}^2 
 \leq 2\ell \max_ {J\in \mathcal I} \rho_J \norm{\wh{c}_J-c_J}^2.
\end{alignat*}
Setting
\[ \Omega_I = \bigcup_{J\in \mathcal I}   \left\{   \norm{\wh{c}_J - c_J}>
  \frac{t}{\sqrt{2\ell\rho_J}}\right\},\]
bound~\eqref{eq:61} gives
\[
\Prob\left[ \Omega_I  \right] \leq 2 \ell \exp\left( -\frac{n
    t^2}{16 \ell \diam{I}^2 }\right),
\]
so that 
\begin{align}
  \label{eq:65}
\Prob\left[ \norm{\wh{w}-v}_2>t \right] \leq 2 \ell \exp\left( -\frac{n
    t^2}{16 \ell \diam{I}^2 }\right).
\end{align}
We now bound the second term in~\eqref{eq:66}. Set
\[ \Omega'_I = \bigcup_{J\in \mathcal C(I)}   \left\{   |
    \sqrt{\frac{n_J}{n}}   -\sqrt{\rho_J}|> 
  \frac{t}{\sqrt{\ell} \diam{I}} \right\},\]
bound~\eqref{eq:41} gives
\[
\Prob\left[ \Omega'_I  \right] \leq 2 \ell \exp\left( -\frac{n
    t^2}{2 \ell \diam{I}^2 }\right),
\]
On the complement on $\Omega_I'$,
\begin{alignat*}{1}
  \norm{\wh{w}-\wh{v} }^2_2 & = \sum_{J\in\mathcal C(I)} 
  \norm{\wh{c}_J-\wh{c}_I}^2 \left(\sqrt{\frac{n_J}{n} } -
    \sqrt{\rho_J}\right)^2 \\
& \leq  \ell  \diam{I}^2 \sup_{J\in \mathcal C(I)} \left|
    \sqrt{\frac{n_J}{n}  } -\sqrt{\rho_J} \right|^2  \leq t^2.
\end{alignat*}
Hence
\begin{align}
  \label{eq:67}
\Prob\left[   \norm{\wh{w}- \wh{v} }_2 >t \right] \leq 2 \ell \exp\left( -\frac{n
    t^2}{2\ell \diam{I}^2 }\right).
\end{align}
Inequality~\eqref{eq:66} with bounds~\eqref{eq:65} and~\eqref{eq:67}
implies~\eqref{eq:64}.
\end{proof}

The next proposition shows that $P_{\Lm}$ is stable under suitable
small perturbations of the partition~$\Lm$.
\begin{proposition}
For any $\eta>0$ 
\begin{subequations}
  \begin{alignat}{1}
    \label{eq:24}
    \Prob\left[ \left(\int\limits_\X \norm{P_{\Lm(\hat{\T}_\eta \cup
            \T_{2\eta})}(x) - P_{\Lm(\hat{\T}_\eta \cap
            \T_{\eta/2})}(x) }^2 d\rho(x)\right)>0 \right] & \lesssim
   (n^\gamma+ \eta^{-\frac{2}{2\sigma+1}}) \exp(-c_a n\eta^2) 
\end{alignat}
and
\begin{alignat}{1}
  \Prob\left[ \left(\int\limits_\X \norm{P_{\Lm(\hat{\T}_\eta \cup
            \T_{2\eta})}(x) - P_{\Lm(\hat{\T}_\eta)}(x) }^2
        d\rho(x)\right)>0 \right]   & \lesssim
    \eta^{-\frac{2}{2\sigma+1}} \exp(-c_a n\eta^2) \label{eq:49},
  \end{alignat}
  where
\begin{align}
  \label{eq:68}
  c_a= \frac{1}{128 (a+1)}.
\end{align}
\end{subequations}
\end{proposition}
\begin{proof}
Recalling the definition of $j_n$ given by~\eqref{eq:57}, set 
$\T^*_n=\bigcup\limits_{j\leq j_n}\Lm_j$, which is a subtree with
\begin{align}
  \label{eq:62}
  \sharp \T^*_n\leq \sum_{j=0}^{j_n} a^j = \frac{a^{j_n+1}-1}{ a-1}\lesssim n^{\gamma},
\end{align}
and, by construction, $\wh{T_\eta}\subset  \T^*_n$.
The probability of the event in the left hand side of~\eqref{eq:24} is
clearly bounded by the probability of the event
\begin{alignat*}{1}
 \set{\hat{\T}_\eta \cap \T_{\eta/2}  \subsetneq \hat{\T}_\eta \cup
      \T_{2\eta}} &  = \bigcup_{I\in\MT}
  \set{ I\in \hat{\T}_\eta \wedge  I\notin \T_{\eta/2}} \cup  \set{
    I\in \T_{2\eta } \wedge  I\notin \hat{\T_{\eta}}}.
\end{alignat*}
About the first term, we observe that
if $I\in
\hat{\T}_\eta\subset\T^*_n$, 
then there exist $k\geq 0$ and $J\in\mathcal C^k(I)\cap \T^*_n$ such 
that $\wh{\epsilon}_J\geq \eta$ and, since   $I\notin \T_{\eta/2}$ and 
$J\in\mathcal C^k(I)$, then $\epsilon_J<\frac{\eta}{2}$, so that
\[ \bigcup_{I\in\T^*_n}
  \set{ I\in \hat{\T}_\eta  \wedge I\notin \T_{\eta/2}} \subset
  \bigcup_{J\in\T^*_n}\set{ \wh{\epsilon}_J\geq \eta\wedge 
   \epsilon_J<\frac{\eta}{2} } \subset 
\bigcup_{J\in\T^*_n}\set{ |\wh{ \epsilon}_J -\epsilon_J| >\frac{\eta}{2} } .
\]
By union bound and~\eqref{eq:64} with $t=\eta/2$,
$\diam{I}\leq 1$  and $\ell\leq a+1$ give
\begin{alignat}{1}
  \Prob\left[ \bigcup_{I\in\T^*_n} \set{\wh{\epsilon}_I\geq \eta\wedge
   \epsilon_I<\frac{\eta}{2}  }\right] &\lesssim  \sharp \T^*_n\,
\exp\left(   -c_a n \eta^2\right)\nonumber\\ 
&\lesssim n^\gamma \exp(-c_a n\eta^2) \label{eq:47},
\end{alignat}
where the second inequality is a consequence of~\eqref{eq:62} and
$c_a$ is given by~\eqref{eq:68}.

By a  similar argument
\[
\set{
    I\in \T_{2\eta } \wedge  I\notin \hat{\T_{\eta}}} \subset 
 \bigcup_{J\in\T_{2\eta} } \set{ |\wh{ \epsilon}_J -\epsilon_J| >\eta }.
\]
By  union bound and~\eqref{eq:64} with $t=\eta$ and
$\diam{I}\leq 1$ give 
\begin{alignat}{1}
  \Prob\left[ \bigcup_{I\in\T_{2\eta}} \set{ \epsilon_I\geq 2\eta\wedge
   \wh{\epsilon}_I<\eta }\right] &\lesssim \sharp {\T_{2\eta}} \,
\exp\left(-c_a n \eta^2 \right)  \lesssim \eta^{-\frac{2}{2\sigma+1}} \exp(-c_a n\eta^2) 
,\label{eq:48}
\end{alignat}
where the second inequality is a consequence of~\eqref{eq:13bis}.

By~\eqref{eq:48} and~\eqref{eq:47}, we get~\eqref{eq:24}. The proof of
\eqref{eq:49} can be deduced reasoning as for~\eqref{eq:48}.
\end{proof}

\section{Auxialiry  results}
 We recall the following probabilistic
inequality based on a result of \cite{pin94,pin99},  see also
\cite[Thm. 3.3.4]{yu} and \cite{pinsak85} for concentration
inequalities for Hilbert-space-valued random variables.
\begin{proposition}\label{pine}
Let $\xi_1,\ldots,\xi_n$ be a family  of independent  zero-mean random
variables taking values in a real separable Hilbert space 
and satisfying 
\begin{align}
  \label{eq:60}
  {\mathbb E}[\norm{\xi_i }^m]\leq
\frac{1}{2} m! \Sigma^2 M^{m-2}\qquad \forall m\geq 2,
\end{align}
where $\Sigma$ and $M$ are two positive constants. Then, for all $n \in\N$
and $t>0$ 
\begin{align}
  \label{92}
   \PP{\norm{\frac{1}{n} \sum_{i=1}^n
\xi_i}\geq t  }\leq  2 \exp\left(-\frac{nt^2
    }{\Sigma^2+M t  +\Sigma\sqrt{\Sigma^2+2M
      t  }} \right)=2
\exp\left(-n\frac{\Sigma^2}{M^2}g(\frac{M t }{\Sigma^2}) \right)
\end{align}
where $g(t) =\frac{t^2}{1+t+\sqrt{1+2t}}$.  In particular, if $\xi_i$
are bounded by $M$  with probability 1, then 
\begin{align}
  \label{92b}
   \PP{\norm{\frac{1}{n} \sum_{i=1}^n
\xi_i}\geq  t  }\leq  2  \exp( - \frac{ nt^2}{4M^2 }).
\end{align}
\end{proposition}
\begin{proof}
Bound\eqref{92} is given in \cite{pin94} with a wrong factor, see
\cite{pin99}. To show\eqref{92b}, note that~\eqref{eq:60} is satisfied
with $\Sigma=M$.  Furthermore, for all $t\leq 1$, $g(t) \geq t^2/4$, so that if
$t\leq M$,
\[\PP{\norm{\frac{1}{n} \sum_{i=1}^n
\xi_i}\geq  t  }\leq  2  \exp( - \frac{ nt^2}{4M^2 })\]
whereas, if $t>M$,~\eqref{92b} is trivially satisfied. 
\end{proof}

The following concentration inequality is based on \cite{mau06} and we adapt
the proof of Theorem~10 in \cite{MaurerP09}. We
introduce the following notation. Given a family $\xi_1,\ldots \xi_n$  of
independent random variables taking value in some measurable space
$\mathcal Y$ and a measurable positive bounded function $f:\mathcal Y^n\to \R$,  for
any $k=1,\ldots,n$ set
  \begin{alignat*}{1}
    V_k & =f(\xi_1,\ldots,\xi_n) - \inf_{y\in\mathcal Y} f(\xi_1,\ldots,\xi_{k-1},y,
    \xi_{k+1},\ldots,\xi_n) \\
        &  = \sup_{y\in\mathcal Y} \Big(f(\xi_1,\ldots,\xi_n) - f(\xi_1,\ldots,\xi_{k-1},y,
    \xi_{k+1},\ldots,\xi_n\Big).
    \end{alignat*}
\begin{proposition}\label{maurer}
With the above notation,  if there exist two constants
$\alpha,\beta> 0$ such that 
\begin{subequations}
  \begin{alignat}{1}
    \max_{k=1,\ldots,n}  V_k &\leq \alpha \label{eq:30} \\
    \label{eq:27}
    \sum_{k=1}^n V_k^2 & \leq \beta f(\xi_1,\ldots,\xi_n)
  \end{alignat}
\end{subequations}
then, for any $t>0$
\begin{align}
  \label{eq:29}
  \Prob\left[ \abs{\sqrt{f(\xi_1,\ldots,\xi_n)} - \sqrt{\mathbb E[f(\xi_1,\ldots,\xi_n)]} }>t\right]
\leq 2 \exp( -\frac{t^2}{2\max\set{\alpha,\beta}}).
\end{align}
\end{proposition}
\begin{proof}
Set $Z_k=V_k/\alpha$ and $Z=f(\xi_1,\ldots,\xi_n)/\alpha $. By
construction
\begin{alignat*}{1}
   Z_k  & \leq 1 \qquad k=1,\ldots,n\\
  \sum_{k=1}^n Z_k^2 & \leq \frac{\beta}{\alpha}  Z.
\end{alignat*}
Let $\gamma=\max\set{\beta/\alpha,1}$. Theorem~13 of \cite{mau06} gives that
\begin{alignat}{1}
   \Prob\left[  \EE{Z} - Z>t\right]& \leq
   \exp(-\frac{t^2}{2 \gamma \mathbb E[Z]} )\label{eq:30a}\\
\Prob\left[  Z -\EE{Z}>t\right]& \leq
   \exp(-\frac{t^2}{2\gamma\mathbb E[Z]+\gamma t} )\label{eq:30b}.
\end{alignat}
By replacing $t$ with $2t \sqrt{\mathbb E[Z]}$ in~\eqref{eq:30a} 
\begin{alignat*}{1}
  \exp(-\frac{2t^2}{\gamma}) & \geq  \Prob\left[  \EE{Z} -2t
    \sqrt{\mathbb E[Z]} +t^2>Z+t^2\right]  
  = \Prob\left[ 
    \abs{\sqrt{\mathbb E[Z]} -t} >\sqrt{Z+t^2}\right] \\&
  \geq \Prob\left[ 
    \sqrt{\mathbb E[Z]}  -\sqrt{Z}>2t\right] 
\end{alignat*}
since
\[
\sqrt{\mathbb E[Z]}  -t \leq \abs{\sqrt{\mathbb E[Z]} -t} \leq
\sqrt{Z+t^2}\leq \sqrt{Z}+t 
\]
provided that  $\abs{ \sqrt{\mathbb E[Z]} -t} \leq \sqrt{Z+t^2} $. Hence
\begin{align}
  \label{eq:31a}
  \Prob\left[ 
    \sqrt{\mathbb E[Z]}  -\sqrt{Z}>t\right] \leq \exp(-\frac{t^2}{2\gamma}) .
\end{align}
Setting $\frac{t^2}{2\mathbb E[Z]+ t}=2 \tau^2$, bound~\eqref{eq:30b} gives
\begin{alignat*}{1}
  \exp(-\frac{2\tau^2}{\gamma}) & \geq  \Prob\left[  Z- \mathbb E[Z]>
    \tau^2+\sqrt{\tau^4+4\tau^2\mathbb E[Z]}\right]  
  \geq \Prob\left[  Z- \mathbb E[Z]
    >4\tau^2 + 4 \tau \sqrt{\mathbb E[Z]}\right] \\
& = \Prob\left[ Z > \left( \sqrt{\mathbb E[Z]}+ 2\tau \right)^2\right] 
=  \Prob\left[ 
    \sqrt{Z}-\sqrt{\mathbb E[Z]}  >2\tau\right] ,
\end{alignat*}
so that, setting $\tau=t/2$,
\begin{align}
  \label{eq:31b}
  \Prob\left[ 
     \sqrt{Z}-\sqrt{\mathbb E[Z]} >t\right] \leq
   \exp(-\frac{t^2}{2\gamma}) .
\end{align}
Bounds~\eqref{eq:31a} and \eqref{eq:31b}  imply that
\[ 
\Prob\left[  \abs{\sqrt{Z}-\sqrt{\mathbb E[Z]} }>t\right] \leq 2
   \exp(-\frac{t^2}{2\gamma}) ,
\]
and, by replacing  $t$ with $t/\sqrt{\alpha}$,
\[
\Prob\left[ 
     \abs{\sqrt{ f(\xi_1,\ldots,\xi_n) }-\sqrt{\mathbb E[f(\xi_1,\ldots,\xi_n) ]}
     }>t\right] \leq 2
   \exp(-\frac{t^2}{2\alpha\gamma}) .
\]
where $\alpha \gamma= \alpha \max\set{\beta/\alpha,1}=
\max\set{\beta,\alpha}$. 
\end{proof}

\bibliography{biblio}

\begin{thebibliography}{10}

\bibitem{allard2012multi}
William~K Allard, Guangliang Chen, and Mauro Maggioni.
\newblock Multi-scale geometric methods for data sets {II}: Geometric
  multi-resolutione analysis.
\newblock {\em Appl. Comput. Harmon. Anal.}, 32(3):435--462, 2012.

\bibitem{arthur2007k}
David Arthur and Sergei Vassilvitskii.
\newblock {k-means++}: The advantages of careful seeding.
\newblock In {\em Proceedings of the eighteenth annual {ACM-SIAM} symposium on
  Discrete algorithms}, pages 1027--1035. Society for Industrial and Applied
  Mathematics, 2007.

\bibitem{belkin2001laplacian}
Mikhail Belkin and Partha Niyogi.
\newblock Laplacian eigenmaps and spectral techniques for embedding and
  clustering.
\newblock In {\em Advances in {N}eural {I}nformation {P}rocessing {S}ystems},
  volume~14, pages 585--591, 2001.

\bibitem{binev2005universal}
Peter Binev, Albert Cohen, Wolfgang Dahmen, Ronald DeVore, and Vladimir
  Temlyakov.
\newblock Universal algorithms for learning theory part {I}: piecewise constant
  functions.
\newblock {\em J. Mach. Learn. Res.}, 6(Sep):1297--1321, 2005.

\bibitem{boluma13}
St{\'e}phane Boucheron, G{\'a}bor Lugosi, and Pascal Massart.
\newblock {\em Concentration inequalities}.
\newblock Oxford University Press, Oxford, 2013.

\bibitem{canas2012learning}
Guillermo Canas, Tomaso Poggio, and Lorenzo Rosasco.
\newblock Learning manifolds with k-means and k-flats.
\newblock In {\em Advances in {N}eural {I}nformation {P}rocessing {S}ystems},
  pages 2465--2473, 2012.

\bibitem{christ}
Michael Christ.
\newblock A {$T(b)$} theorem with remarks on analytic capacity and the {C}auchy
  integral.
\newblock {\em Colloq. Math.}, 60/61(2):601--628, 1990.

\bibitem{donoho}
David~L. Donoho and Carrie Grimes.
\newblock Hessian eigenmaps: Locally linear embedding techniques for
  high-dimensional data.
\newblock {\em Proc. Natl. Acad. Sci. USA}, 100(10):5591--5596, 2003.

\bibitem{gile17}
Giacomo Gigante and Paul Leopardi.
\newblock Diameter bounded equal measure partitions of {A}hlfors regular metric
  measure spaces.
\newblock {\em Discrete Comput. Geom.}, 57(2):419--430, 2017.

\bibitem{gromov}
Misha Gromov.
\newblock {\em Metric structures for {R}iemannian and non-{R}iemannian spaces}.
\newblock Birkh\"{a}user Boston, Inc., Boston, MA, english edition, 2007.

\bibitem{gruber2004optimum}
Peter~M Gruber.
\newblock Optimum quantization and its applications.
\newblock {\em Adv. Math.}, 186(2):456--497, 2004.

\bibitem{hastie}
Trevor Hastie, Robert Tibshirani, and Jerome Friedman.
\newblock {\em The Elements of Statistical Learning}.
\newblock Springer New York Inc., 2001.

\bibitem{hastie2013introduction}
Gareth James, Daniela Witten, Trevor Hastie, and Robert Tibshirani.
\newblock {\em An introduction to statistical learning}, volume 112.
\newblock Springer, 2013.

\bibitem{lima16}
Wenjing Liao and Mauro Maggioni.
\newblock Adaptive geometric multiscale approximations for intrinsically
  low-dimensional data.
\newblock {\em J. Mach. Learn. Res.}, 20(98):1--83, 2019.

\bibitem{lloyd1982least}
Stuart Lloyd.
\newblock Least squares quantization in {PCM}.
\newblock {\em IEEE Trans. Inform. Theory}, 28(2):129--137, 1982.

\bibitem{maggioni2014dictionary}
Mauro Maggioni, Stanislav Minsker, and Nate Strawn.
\newblock Dictionary learning and non-asymptotic bounds for geometric
  multi-resolution analysis.
\newblock {\em PAMM. Proc. Appl. Math. Mech.}, 14(1):1013--1016, 2014.

\bibitem{mallat}
Stphane Mallat.
\newblock {\em A Wavelet Tour of Signal Processing, Third Edition: The Sparse
  Way}.
\newblock Academic Press, Inc., 3rd edition, 2008.

\bibitem{mau06}
Andreas Maurer.
\newblock Concentration inequalities for functions of independent variables.
\newblock {\em Random Structures Algorithms}, 29(2):121--138, 2006.

\bibitem{MaurerP09}
Andreas Maurer and Massimiliano Pontil.
\newblock Empirical {B}ernstein bounds and sample-variance penalization.
\newblock In {\em COLT}, 2009.

\bibitem{nadler2006diffusion}
Boaz Nadler, St{\'e}phane Lafon, Ronald~R Coifman, and Ioannis~G Kevrekidis.
\newblock Diffusion maps, spectral clustering and reaction coordinates of
  dynamical systems.
\newblock {\em Appl. Comput. Harmon. Anal.}, 21(1):113--127, 2006.

\bibitem{olshausen}
Bruno Olshausen and David Field.
\newblock Emergence of simple-cell receptive field properties by learning a
  sparse code for natural images.
\newblock {\em Nature}, 381(607):6583, 1996.

\bibitem{petersen}
Peter Petersen.
\newblock {\em Riemannian geometry}.
\newblock Springer, Cham, third edition, 2016.

\bibitem{pinsak85}
I.~F. Pinelis and A.~I. Sakhanenko.
\newblock Remarks on inequalities for probabilities of large deviations.
\newblock {\em Theory Probab. Appl.}, 30(1):143--148, 1985.

\bibitem{pin94}
Iosif Pinelis.
\newblock Optimum bounds for the distributions of martingales in {B}anach
  spaces.
\newblock {\em Ann. Probab.}, 22(4):1679--1706, 1994.

\bibitem{pin99}
Iosif Pinelis.
\newblock Correction: ``{O}ptimum bounds for the distributions of martingales
  in {B}anach spaces'' [{A}nn.\ {P}robab.\ {\bf 22} (1994), no. 4, 1679--1706;
  {MR}1331198 (96b:60010)].
\newblock {\em Ann. Probab.}, 27(4):2119, 1999.

\bibitem{kpca}
Bernhard Sch\"{o}lkopf, Alexander~J. Smola, and Klaus-Robert M\"{u}ller.
\newblock Advances in kernel methods.
\newblock In Bernhard Sch\"{o}lkopf, Christopher J.~C. Burges, and Alexander~J.
  Smola, editors, {\em International conference on artificial neural networks},
  chapter Kernel Principal Component Analysis, pages 327--352. MIT Press,
  Cambridge, MA, USA, 1999.

\bibitem{schwartz}
Laurent Schwartz.
\newblock {\em Cours d'analyse. 3}.
\newblock Hermann, Paris, second edition, 1993.

\bibitem{scott2006minimax}
Clayton Scott, Robert~D Nowak, et~al.
\newblock Minimax-optimal classification with dyadic decision trees.
\newblock {\em IEEE Trans. Inform. Theory}, 52(4):1335--1353, 2006.

\bibitem{tenenbaum}
Josh Tenenbaum, Vin De~Silva, and John Langford.
\newblock A global geometric framework for nonlinear dimensionality reduction.
\newblock {\em Science}, 290(5500):2319--2323, 2000.

\bibitem{ward1963hierarchical}
Joe~H Ward~Jr.
\newblock Hierarchical grouping to optimize an objective function.
\newblock {\em J. Amer. Statist. Assoc.}, 58(301):236--244, 1963.

\bibitem{yu}
V.~Yurinsky.
\newblock {\em Sums and {G}aussian vectors}, volume 1617.
\newblock Springer-Verlag, Berlin, 1995.

\end{thebibliography}
\bibliographystyle{plain}    

\end{document}